\documentclass{article}

\newcommand{\METHOD}{{MURAL}\xspace}

\usepackage{algorithm,amsfonts,amsmath,amssymb,bm,booktabs,changes,dsfont,enumerate,fancyhdr,graphicx,hyperref,microtype,nicefrac,url,wrapfig,xcolor,xspace}

\usepackage{microtype}
\usepackage{graphicx}
\usepackage{subcaption}
\usepackage{booktabs} 
\usepackage{amsmath}
\usepackage{amsthm}
\newtheorem{theorem}{Theorem}

\DeclareMathOperator*{\argmax}{arg\,max}

\usepackage{hyperref}
\usepackage{cleveref}

\usepackage[accepted]{icml2021}

\icmltitlerunning{MURAL: Meta-Learning Uncertainty-Aware Rewards for RL}

\begin{document}

\twocolumn[
\icmltitle{MURAL: Meta-Learning Uncertainty-Aware Rewards for Outcome-Driven Reinforcement Learning}



\icmlsetsymbol{equal}{*}

\begin{icmlauthorlist}
\icmlauthor{Kevin Li}{equal,ucb}
\icmlauthor{Abhishek Gupta}{equal,ucb}
\icmlauthor{Ashwin Reddy}{ucb}
\icmlauthor{Vitchyr Pong}{ucb}
\icmlauthor{Aurick Zhou}{ucb}
\icmlauthor{Justin Yu}{ucb}
\icmlauthor{Sergey Levine}{ucb}
\end{icmlauthorlist}

\icmlaffiliation{ucb}{Department of Electrical Engineering and Computer Sciences, UC Berkeley, Berkeley, USA}

\icmlcorrespondingauthor{Kevin Li}{kevintli@berkeley.edu}
\icmlcorrespondingauthor{Abhishek Gupta}{abhigupta@berkeley.edu}

\icmlkeywords{Machine Learning, Reinforcement Learning, Bayesian Classifiers, Normalized Maximum Likelihood, Robotics, ICML}

\vskip 0.3in
]



\printAffiliationsAndNotice{\icmlEqualContribution} 

\begin{abstract}
Exploration in reinforcement learning is a challenging problem: in the worst case, the agent must search for high-reward states that could be hidden anywhere in the state space. Can we define a more tractable class of RL problems, where the agent is provided with examples of successful outcomes? In this problem setting, the reward function can be obtained automatically by training a classifier to categorize states as successful or not. If trained properly, such a classifier can provide a well-shaped objective landscape that both promotes progress toward good states and provides a calibrated exploration bonus. In this work, we show that an uncertainty aware classifier can solve challenging reinforcement learning problems by both encouraging exploration and provided directed guidance towards positive outcomes. We propose a novel mechanism for obtaining these calibrated, uncertainty-aware classifiers based on an amortized technique for computing the normalized maximum likelihood (NML) distribution. To make this tractable, we propose a novel method for computing the NML distribution by using meta-learning. We show that the resulting algorithm has a number of intriguing connections to both count-based exploration methods and prior algorithms for learning reward functions, while also providing more effective guidance towards the goal. We demonstrate that our algorithm solves a number of challenging navigation and robotic manipulation tasks which prove difficult or impossible for prior methods.
\end{abstract}

\section{Introduction}
\label{sec:intro}
While reinforcement learning (RL) has been shown to successfully solve problems with careful reward design~\citep{rajeswaran2018learning}, RL in its most general form, with no assumptions on the dynamics or reward function, requires solving a challenging uninformed search problem in which rewards are sparsely observed. Techniques that explicitly provide ``reward-shaping''~\citep{ng99rewardshaping}, or modify the reward function to guide learning, can help take some of the burden off of exploration, but shaped rewards are often difficult to provide without significant domain knowledge. Moreover, in many domains of practical significance, actually specifying rewards in terms of high dimensional observations can be extremely difficult, making it difficult to directly apply RL to problems with challenging exploration.

\begin{figure}[t]
    \centering
    \includegraphics[width=\linewidth]{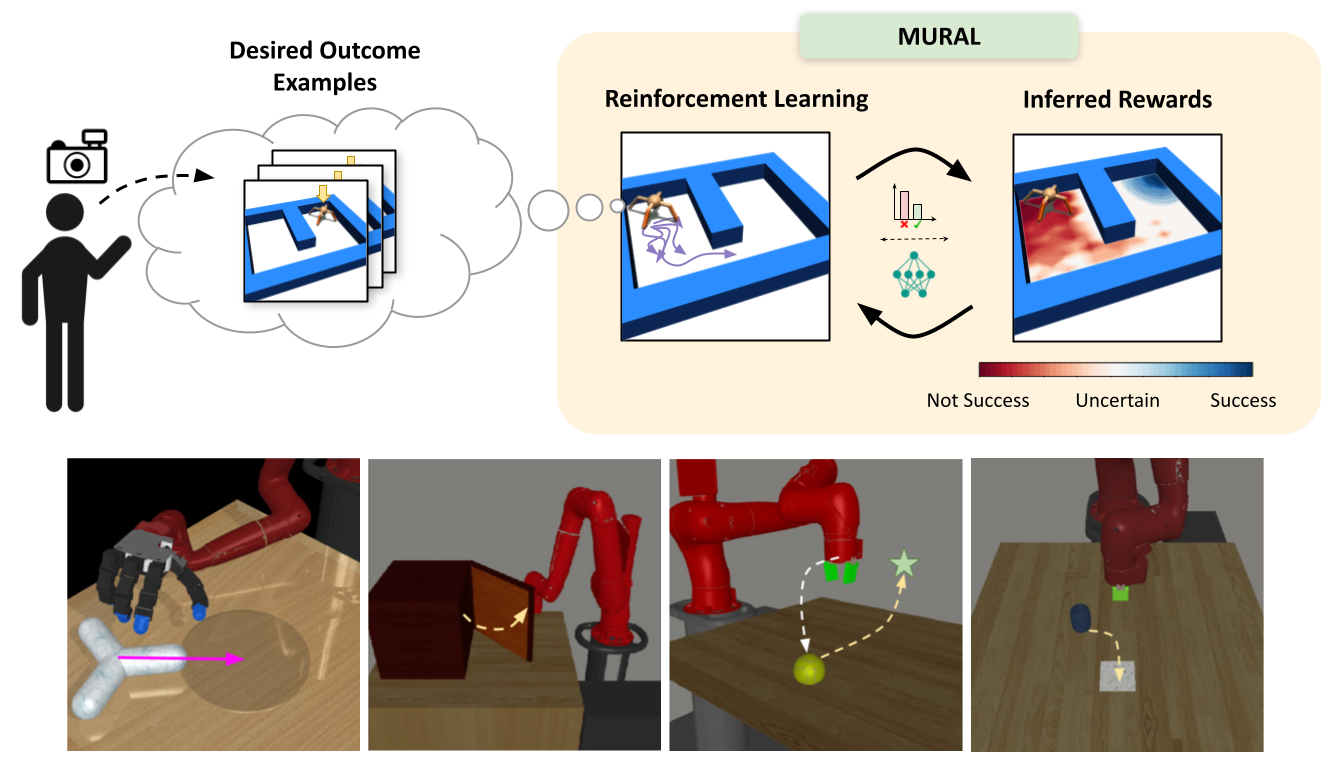}
    \caption{\footnotesize{\textbf{\METHOD:} Our method trains an uncertainty-aware classifier based on user-provided examples of successful outcomes. Appropriate uncertainty in the classifier, obtained via a meta-learning based estimator for the normalized maximum likelihood (NML) distribution, automatically incentivizes exploration and provides reward shaping for RL. This can solve complex robotic manipulation and navigation tasks as shown here.}}
    \vspace{-0.5cm}
    \label{fig:main-fig}
\end{figure}

Can the RL problem be made more tractable if the agent is provided with examples of successful outcomes instead of an uninformative reward function? Such examples are often easier to provide than, for example, entire demonstrations or a hand-designed shaped reward function. However, they can still provide considerable guidance on how to successfully accomplish a task, potentially alleviating exploration challenges if the agent can successfully recognize similarities between visited states and the provided examples. In this paper, we study such a problem setting, where instead of a hand-designed reward function, the RL algorithm is provided with a set of \emph{successful outcome examples}: states in which the desired task has been accomplished successfully. Prior work~\citep{vice18fu,r3l} aims to solve tasks in this setting by estimating the distribution over these states and maximizing the probability of reaching states that are likely under this distribution. While this can work well in some domains, it has largely been limited to settings without significant exploration challenges. In our work, we focus on the potential for this mode of task specification to enable RL algorithms to solve more challenging tasks without the need for manual reward shaping. Intuitively, the availability of extra information in the form of explicit success examples can provide the algorithm more directed information for exploration, rather than having to simply rely on uninformed task agnostic exploration methods. This allows us to formulate a class of more tractable problems, which we refer to as outcome-driven RL.

However, in order to attain improved exploration, an outcome-driven RL agent must be able to estimate some notion of similarity between the visited states and successful outcomes, so as to utilize this similarity as an automatic reward shaping. Our method addresses this challenge by training a classifier to distinguish successful states, provided by the user, from those generated by the current policy, analogously to generative adversarial networks~\citep{gan} and previously proposed methods for inverse reinforcement learning~\citep{fu2018learning}. In general, such a classifier may not provide effective reward shaping for learning the policy, since it does not explicitly quantify uncertainty about success probabilities and can be overly pessimistic in providing reward signal for learning. We discuss how Bayesian classifiers incorporating a particular form of uncertainty quantification based on the normalized maximum likelihood (NML) distribution can incentivize exploration in outcome-driven RL problems. To understand its benefits, we connect our approach to count-based exploration methods, while also showing that it improves significantly over such methods when the classifier exhibits good generalization properties, due to its ability to utilize success examples. Finally, we propose a practical algorithm to train NML-based success classifiers in a computationally efficient way using meta-learning, and show experimentally that our method can more effectively solve a range of challenging navigation and robotic manipulation tasks. 

Concretely, this work illustrates the challenges of using standard success classifiers~\citep{vice18fu} for outcome-driven RL, and proposes a novel technique for training uncertainty aware classifiers with normalized maximum likelihood, which is able to both incentivize the exploration of novel states and provide reward shaping that guides exploration towards successful outcomes. We present a tractable algorithm for learning these uncertainty aware classifiers in practice by leveraging concepts from meta-learning. We analyze our proposed technique for reward inference experimentally across a number of navigation and robotic manipulation domains and show benefits over prior classifier-based RL methods as well as goal-reaching methods.

\section{Related Work}
\label{sec:related}

While a number of methods have been proposed to improve exploration, it remains a challenging open problem in RL~\cite{misra2019homer}. Standard exploration methods either add bonuses to the reward function that encourage a policy to visit novel states in a task-agnostic manner~\citep{wiering1998efficient,auer2002finite,schaul2011curiosity, houthooft2016vime,fu2017ex2,pathak2017curiosity,tang2017exploration,stadie2015incentivizing,bellemare2016unifying,burda2018exploration,odonoghue2018variational}, or approximate Thompson sampling from a posterior over value functions~\citep{strens2000bayesian,osband2013more,osband2016deep}. Whereas these techniques are uninformed about the task, we consider a constrained but widely applicable setting where the desired outcome can be specified by success examples, allowing for more efficient task-directed exploration.

Designing well-shaped reward functions can also make exploration easier, but often requires significant domain knowledge ~\citep{andrychowicz2020learning}, access to privileged information~\citep{levine2016gps} or a human in the loop providing rewards~\citep{knox2009interactively,singh2019end}.
Prior work has considered specifying rewards by providing example demonstrations and inferring rewards with inverse RL~\citep{abbeel2004apprenticeship,ziebart2008maximum,ho2016generative,fu2018learning}. This requires expensive expert demonstrations to be provided to the agent. In contrast, our work has the minimal requirement of successful outcome states, which can be provided more cheaply and intuitively. This subclass of problems is also related to goal-conditioned RL ~\citep{kaelbling1993learning,schaul2015universal,zhu2017target,andrychowicz2017hindsight,nair2018visual,veeriah2018many,rauber2018hindsight,warde2018unsupervised,colas2019curious,ghosh2019learning,pong2020skew} but is more general, since it allows for the notion of success to be more abstract than reaching a single state. 

A core idea behind our method is using a Bayesian classifier to learn a suitable reward function. Bayesian inference with expressive models and high dimensional data can often be intractable, requiring strong assumptions on the form of the posterior~\citep{hoffman2013stochastic,blundell2015weight,maddox2019simple}. In this work, we build on the concept of normalized maximum likelihood~\citep{rissanen1996fisher,shtar1987universal}, or NML, to learn Bayesian classifiers that can impose priors over the space of outcomes. Although NML is typically considered from the perspective of optimal coding~\citep{grunwald2007minimum,fogel2018universal}, we show how it can be used for success classifiers, and discuss connections to exploration in RL. We propose a novel technique for making NML computationally tractable based on meta-learning, which more directly optimizes for quick NML computation as compared to prior methods like ~\citet{zhou2020acnml} which learn an amortized posterior.

\section{Preliminaries}
\label{sec:prelim}

In this section, we discuss background on RL using successful outcome examples as well as conditional normalized maximum likelihood.

\subsection{Reinforcement Learning with Outcome Examples}
\label{sec:prelim-rl}

\begin{figure*}[t]
    \centering
    \includegraphics[width=0.95\linewidth]{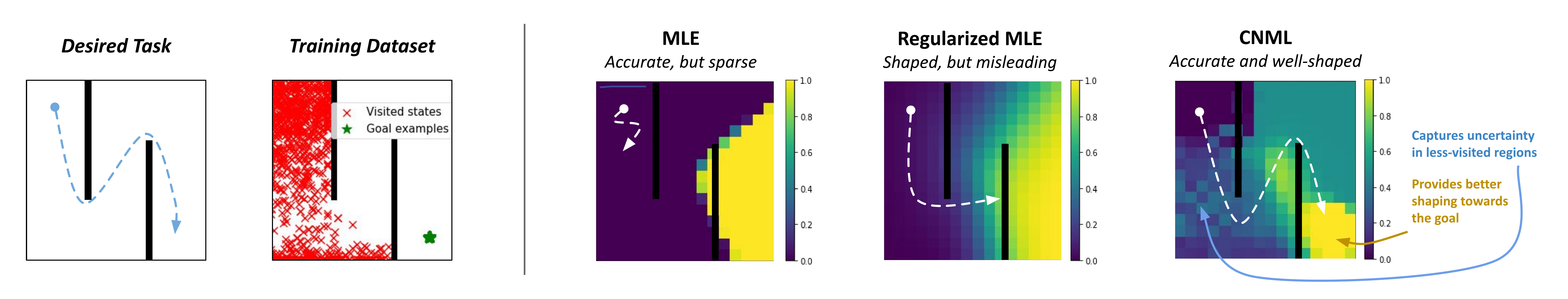}
    \vspace{-0.4cm}
    \caption{\footnotesize{Comparison of rewards given by various classifier training schemes on a 2D maze example. Typical maximum likelihood (MLE) classifiers commonly suffer from either a lack of useful learning signal (if trained to convergence) or misleading local optima (if regularized using standard methods such as weight decay or early stopping), whereas CNML produces accurate and well-shaped rewards.}}
    \label{fig:failclassifier}
\end{figure*}

We follow the framework proposed by \citet{vice18fu} and assume that we are provided with a Markov decision process (MDP) \emph{without} a reward function, given by $\mathcal{M} = (\mathcal{S}, \mathcal{A}, \mathcal{T}, \gamma, \mu_0)$, as well as successful outcome examples $\mathcal{S}_{+} = \{s_+^k\}_{k=1}^K$, which is a set of states in which the desired task has been accomplished. This formalism is easiest to describe in terms of the control as inference framework~\citep{cai}. The relevant graphical model (refer to ~\citet{vice18fu}) consists of states and actions, as well as binary success variables $e_t \in \{0, 1\}$ that represent the occurrence of a particular event. The agent's objective is to cause this event to occur (e.g., a robot that is cleaning the floor must cause the ``floor is clean'' event to occur). Formally, we assume that the states in $\mathcal{S}_{+}$ are sampled from the distribution $p(s_t | e_t = 1)$ — that is, states where the desired event has taken place — and try to infer the distribution $p(e_t = 1 | s_t)$ to use as a reward function. In this work, we focus on efficient methods for solving this reformulation of the RL problem by utilizing a novel uncertainty quantification method to represent $p(e_t | s_t)$.

In practice, prior methods that build on this formulation of the RL problem derive an algorithm where the reward function in RL is produced by a classifier that estimates $p(e_t = 1 | s_t)$. Following the derivation in ~\citet{fu2018learning}, it is possible to show that the correct source of \emph{negative} examples is the state distribution of the policy itself, $\pi(s)$. This insight results in a simple  algorithm: at each iteration of the algorithm, the policy is updated to maximize the current reward, given by $\log p(e_t = 1 | s_t)$, then samples from the policy are added to the set of negative examples $\mathcal{S}_{-}$, and the classifier is retrained on the original positive set $\mathcal{S}_{+}$ and the updated negative set $\mathcal{S}_{-}$. As noted by \citet{vice18fu}, this process is closely connected to GANs and inverse reinforcement learning, where the classifier plays the role of the discriminator and the policy that of the generator. However, as we will discuss, this strategy can often face significant exploration challenges.  

\subsection{Conditional Normalized Maximum Likelihood}
\label{sec:prelim-nml}
Our work utilizes the principle of conditional normalized maximum likelihood (CNML)~\citep{Rissanen2007ConditionalModels, grunwald2007minimum, fogel2018universal}, which provides a technique for learning uncertainty aware, calibrated classifiers. We review the key ideas behind CNML briefly.  CNML is a method for performing $k$-way classification, given a model class $\Theta$ and a dataset $\mathcal{D} = \{(x_0, y_0), (x_1, y_1), ..., (x_n, y_n)\}$, and has been shown to provide better calibrated predictions and uncertainty estimates with minimax regret guarantees~\citep{pnml}. More specifically, the CNML distribution can be shown to provably minimize worst-case regret against an oracle learner that has access to the true labels, but does not know which point it will be tested on. We refer the reader to ~\citet{fogel2018universal, zhou2020acnml} for a more complete consideration of the theoretical properties of the CNML distribution.

To predict the class of a query point $x_q$, CNML constructs $k$ augmented datasets by adding $x_q$ with a different label in each dataset, which we write as $\mathcal{D} \cup (x_q, y=i), i \in (1, 2, ..., k)$.
CNML then defines the class distribution by solving the maximum likelihood estimation problem at query time for each of these augmented datasets to convergence, and normalizes the likelihoods as follows:
\begin{align}
    &p_{\text{CNML}}(y=i|x_q) = \frac{p_{\theta_i}(y=i|x_q)}{\sum_{j=1}^k p_{\theta_j}(y=j|x_q)} \\
    &\theta_i = \argmax_{\theta \in \Theta} \mathbb{E}_{(x, y) \sim \mathcal{D} \cup (x_q, y=i)}[\log p_{\theta}(y|x)]
    \label{eqn:CNML}
\end{align}
If $x_q$ is close to other datapoints in $\mathcal{D}$, the model will struggle to assign a high likelihood to labels that differ substantially from those of nearby points.
However, if $x_q$ is far from all datapoints in $\mathcal{D}$, then the different augmented maximum likelihood problems can easily classify $x_q$ as any arbitrary class, providing us with likelihoods closer to uniform. We refer readers to \citet{grunwald2007minimum} for an in-depth discussion of CNML and its connections to minimum description length and regret minimization. Intuitively, the CNML classifier provides a way to impose a uniform prior for uncertainty quantification, where we predict the uniform distribution on unseen inputs since they are maximally uncertain, and defer more to the maximum likelihood solution on frequently seen inputs since they are minimally uncertain.

\section{Bayesian Success Classifiers for Reward Inference}
\label{sec:method}

As discussed in Section~\ref{sec:prelim-rl}, a principled way of approaching outcome-driven RL is to train a classifier to determine whether a particular state is a successful outcome or not. However, while such a technique would eventually converge to the correct solution, it frequently suffers from uninformative or incorrect rewards during the learning process. For example, Figure~\ref{fig:failclassifier} depicts a simple 2D maze scenario where the agent starts at the top left corner and the positive outcomes are at the bottom right corner of the environment. Without suitable regularization, the decision boundary may take on the form of a sharp boundary \emph{anywhere} between the positive and negative examples in the early stages of training. As a result, the classifier might provide little to no reward signal for the policy, since it can assign arbitrarily small probabilities to the states sampled from the policy. Given that the classifier-based RL process is essentially equivalent to training a GAN (as described in Section~\ref{sec:prelim-rl}), this issue is closely related to the challenges of GAN training as noted by ~\citet{arjovsky18ganprincipled}, where an ideal maximum likelihood discriminator provides no gradient signal for training the generator.  

We note that this issue is not pathological: our experiments in \Cref{sec:experiments} show that this phenomenon of poor reward shaping happens in practice. In addition, introducing na\"ively chosen forms of regularization such as weight decay, as is common in prior works, may actually provide \emph{incorrect} reward shaping to the algorithm, making it more challenging to actually accomplish the task (as illustrated in Figure~\ref{fig:failclassifier}). This often limits classifier-based RL techniques to tasks with trivial exploration challenges. In this section, we will discuss how a simple change to the procedure for training a classifier, going from standard maximum likelihood estimation to an approach based on the principle of normalized maximum likelihood, allows for an appropriate consideration of uncertainty quantification that can solve problems with non-trivial exploration challenges.

\subsection{Regularized Success Classifiers via Normalized Maximum Likelihood}
\label{sec:nml_classifier}

\begin{figure}
\vspace{-0.3cm}
\begin{minipage}{\linewidth}
    \begin{algorithm}[H]
      	\caption{RL with CNML-Based Success Classifiers}
      	\label{alg:cnml-alg}
      	\begin{algorithmic}[1]
      	\STATE User provides success examples $\mathcal{S}_{+}$
      	 \STATE Initialize policy $\pi$, replay buffer $\mathcal{S}_{-}$, and reward classifier parameters $\theta_{\mathcal{R}}$
      	\FOR{iteration $i=1, 2, ...$}
      	    \STATE Add on-policy examples to $\mathcal{S}_{-}$ by executing $\pi$.
            \STATE Sample $n_\text{test}$ points from $\mathcal{S}_{+}$ (label 1) and $n_\text{test}$ points from $\mathcal{S}_{-}$ (label 0) to construct a dataset $\mathcal{D}$
            \STATE Assign state rewards as $r(s) = p_\text{CNML}(e=1|s, \mathcal{D})$
            \STATE Train $\pi$ with RL algorithm
      	\ENDFOR
      	\end{algorithmic}
    \end{algorithm}
\end{minipage}
\vspace{-0.6cm}
\end{figure}

It is important to note that for effective exploration in reinforcement learning, the rewards should not just indicate whether a state is a successful outcome (since this will be $0$ everywhere but successful outcomes), but should instead provide a sense of whether a particular state may be on the path to a successful outcome and should be explored further. The standard maximum likelihood classifier described in Section~\ref{sec:prelim} is overly pessimistic in doing so, setting the likelihood of all intermediate states to $0$ in the worst case, potentially mislabeling promising states to explore. To avoid this, we want to use a classification technique that minimizes this worst-case regret, maintaining some level of uncertainty about whether under-visited states are on the path to successful outcomes. As discussed in Section~\ref{sec:prelim-nml},  the technique of conditional normalized maximum likelihood provides us a straightforward way to obtain such a classifier. CNML is particularly well suited to this problem since, as discussed by \citet{Zhang2011ModelSW}, it essentially imposes a uniform prior over the space of outcomes. It thus avoids pathological collapse of rewards by maintaining a measure of uncertainty over whether a state is potentially promising to explore further, rather than immediately bringing its likelihood to $0$ as maximum likelihood solutions would.

To use CNML for reward inference, the procedure is similar to the one described in Section~\ref{sec:prelim}. We construct a dataset using the provided successful outcomes as positives and on-policy samples as negatives. However, the label probabilities for RL are instead produced by the CNML procedure to obtain rewards $r(s) = p_{\text{CNML}}(e=1|s)$ as follows: 

\begin{align}
    &r(s) = \frac{p_{\theta_1}(e=1|s)}{ p_{\theta_1}(e=1|s) + p_{\theta_0}(e=0|s)} \\
    &\theta_0 = \argmax_{\theta \in \Theta} \mathbb{E}_{(s_j, e_j) \sim \mathcal{D} \cup (s, e=0)}[\log p_{\theta}(e_j|s_j)] \\
    &\theta_1 = \argmax_{\theta \in \Theta} \mathbb{E}_{(s_j, e_j) \sim \mathcal{D} \cup (s, e=1)}[\log p_{\theta}(e_j|s_j)]
    \label{eqn:CNML-RL}
\end{align}

This reward is then used to perform policy updates, new data is collected with the updated policy, and the process is repeated. A full description can be found in Algorithm \ref{alg:cnml-alg}.

To illustrate how this change affects reward assignment during learning, we visualize a potential assignment of rewards with a CNML-based classifier on the problem described earlier in Fig~\ref{fig:failclassifier}. When the success classifier is trained with CNML instead of standard maximum likelihood, intermediate unseen states would receive non-zero rewards rather than simply having vanishing likelihoods like the maximum likelihood solution, thereby incentivizing exploration. In fact, the CNML likelihood has a strong connection to count-based exploration, as we show next. Additionally, we also see that CNML is able to provide more directed shaping towards the successful outcomes when generalization exists across states, as explained below. 

\subsection{Relationship to Count-Based Exploration} 
In this section we relate the success likelihoods obtained via CNML to commonly used exploration methods based on counts.  Formally, we prove that the success classifier trained with CNML is equivalent to a version of count-based exploration in the absence of any generalization across states (i.e., a fully tabular setting).

\begin{theorem}
\label{cnml-counts}
Suppose we are estimating success probabilities $p(e=1\vert s)$ in the tabular setting, where we have an independent parameter for each state.
Let $N(s)$ denote the number of times state $s$ has been visited by the policy, and let $G(s)$ be the number of occurrences of state $s$ in the set of positive examples.
Then the CNML success probability $p_{\text{CNML}}(e=1\vert s)$ is equal to $\frac{G(s)+1}{N(s)+G(s)+2}$.
For states that are not represented in the positive examples, i.e. G(s) = 0, we then recover inverse counts $\frac{1}{N(s) + 2}.$
\end{theorem}
Refer to Appendix A.1
for a full proof. While CNML has a strong connection with counts as described above, it is important to note two advantages. First, the rewards are estimated without an explicit generative model, simply by using a standard discriminative model trained via CNML. Second, in the presence of generalization via function approximation, the exploration behavior from CNML can be significantly more task directed, as described next.  

In most problems, when the classifier is parameterized by a function approximator with non-trivial generalization, the structure of the state space actually provides more information to guide the agent towards the successful examples than simply using counts. In most environments~\citep{openaigym, yu2019metaworld} states are not completely uncorrelated, but instead lie in a representation space where generalization correlates with the dynamics structure in the environment. For instance, states from which successful outcomes can be reached more easily (i.e., states that are ``close'' to successful outcomes) are likely to have similar representations. Since the uncertainty-aware classifier described in Section~\ref{sec:nml_classifier} is built on top of such features and is trained with knowledge of the desired successful outcomes, it is able to incentivize more \textit{task-aware} directed exploration than simply using counts. This phenomenon is illustrated intuitively in Fig~\ref{fig:failclassifier}, and demonstrated empirically in our experimental analysis in Section~\ref{sec:experiments}.

\section{\METHOD: Training Uncertainty-Aware Success Classifiers for Outcome Driven RL via Meta-Learning and CNML}
\label{sec:meta-nml}
In Section~\ref{sec:method}, we discussed how success classifiers trained via CNML can incentivize exploration and provide reward shaping to guide RL. However, the reward inference technique via CNML described in Section~\ref{sec:nml_classifier} is in most cases computationally intractable, as it requires optimizing separate maximum likelihood estimation problems to convergence on every data point we want to query. In this section, we describe a novel approximation that allows us to apply this method in practice.

\begin{figure*}[t]
    \centering
    \includegraphics[width=0.9\textwidth]{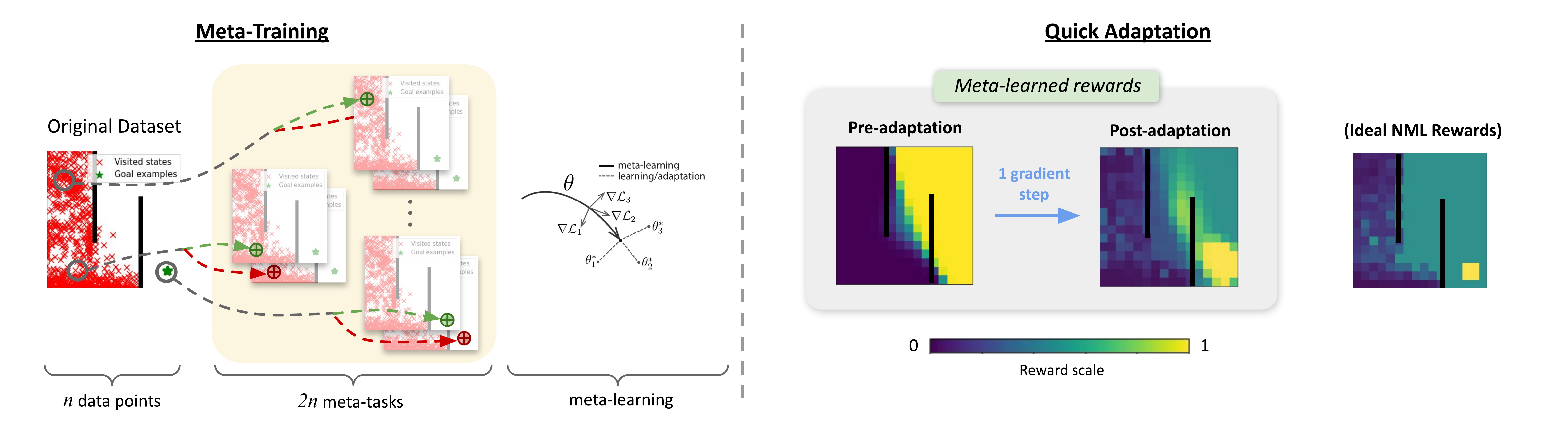}
    \caption{\footnotesize{Diagram of using meta-NML to train a classifier. Meta-NML learns an initialization that can quickly adapt to new datapoints with arbitrary labels. At evaluation time, it approximates the NML probabilities (right) fairly well with a single gradient step.}}
    \label{fig:meta-nml-flow}
    \vspace{-0.25cm}
\end{figure*}

\subsection{Meta-Learning for CNML}
We adopt ideas from meta-learning to amortize the cost of obtaining the CNML distribution. As noted in Section~\ref{sec:nml_classifier}, computing the CNML distribution involves repeatedly solving maximum likelihood problems. While computationally daunting, these problems share a significant amount of common structure, which we can exploit to estimate the CNML distribution more efficiently. Meta-learning uses a distribution of training problems to explicitly meta-train models that can quickly adapt to new problems and, as we show next, can be directly used to accelerate CNML.

To apply meta-learning to computing the CNML distribution, we can formulate each of the maximum likelihood problems described in Equation~\ref{eqn:CNML} as a separate task for meta-learning, and apply a standard meta-learning technique to obtain a model capable of few-shot adaptation to the maximum likelihood problems required for CNML. While any meta-learning algorithm is applicable, we found model agnostic meta-learning (MAML)~\cite{finn17icml} to be effective. MAML aims to meta-train a model that can quickly adapt to new tasks via a few steps of gradient descent by explicitly performing a bi-level optimization. We refer readers to ~\citet{finn17icml} for a detailed overview.

The meta-training procedure to enable quick querying of CNML likelihoods can be described as follows. Given a dataset $\mathcal{D} = \{(\mathbf{x_0}, \mathbf{y_0}), (\mathbf{x_1}, \mathbf{y_1}), ..., (\mathbf{x_n}, \mathbf{y_n})\}$, we construct $2n$ different tasks $\tau_i$, each corresponding to performing maximum likelihood estimation on the original dataset $\mathcal{D}$ combined with an additional point $(\mathbf{x_i}, \mathbf{y'})$, where $y'$ is a proposed label of either 0 or 1 and $\mathbf{x_i}$ is a point from the dataset $\mathcal{D}$. Given these constructed tasks $\mathcal{S}(\tau)$, we perform meta-training as described by ~\citet{finn17icml}:

\begin{align}
\vspace{-0.3cm}
&\max_{\theta} \hspace{0.1cm} \mathbb{E}_{\mathbf{x_i} \sim \mathcal{D}, y'. \in \{0, 1\}}[\mathcal{L}(\mathcal{D} \cup (\mathbf{x_i}, \mathbf{y'}) , \theta_i')],\hspace{0.4cm} \\ &s.t \hspace{0.2cm} \theta_i' = \theta - \alpha \nabla_{\theta} \mathcal{L}(\mathcal{D} \cup (\mathbf{x_i}, \mathbf{y'}), \theta).
\label{eqn:meta-nml-training}
\vspace{-0.1cm}
\end{align}
This training procedure produces a set of parameters $\theta$ that can then be quickly adapted to provide the CNML distribution with a step of gradient descent. The model can be queried for the CNML distribution by starting from the meta-learned $\theta$ and taking one step of gradient descent for the dataset augmented with the query point, each with a different potential label. These likelihoods are then normalized to provide the CNML distribution as follows:
\begin{align}
&p_{\text{meta-NML}}(y|x; \mathcal{D}) = \frac{p_{\theta_y}(y|x)}{\sum_{y \in \mathcal{Y}}p_{\theta_y}(y|x)} 
\\ &\theta_y = \theta - \alpha \nabla_\theta \mathbb{E}_{(x_i, y_i) \sim \mathcal{D} \cup (x, y)}[\mathcal{L}(x_i, y_i, \theta)].
\label{eqn:meta-nml-testing}
\end{align}
This process is illustrated in Fig~\ref{fig:meta-nml-flow}, which shows how the meta-NML procedure can be used to obtain approximate CNML likelihoods with just a single gradient step. 

This scheme for amortizing the computational challenges of NML (which we call \emph{meta-NML}) allows us to obtain normalized likelihood estimates without having to retrain maximum likelihood to convergence at every single query point. A complete description, runtime analysis and pseudocode of this algorithm are provided in Appendix A.2 and A.3. Crucially, we find that meta-NML is able to approximate the CNML outputs well with just one or a few gradient steps, making it several orders of magnitude faster than standard CNML.

\begin{figure*}[!h]
\centering
    \begin{subfigure}[b]{0.13\textwidth}
        \centering
        \includegraphics[width=\textwidth]{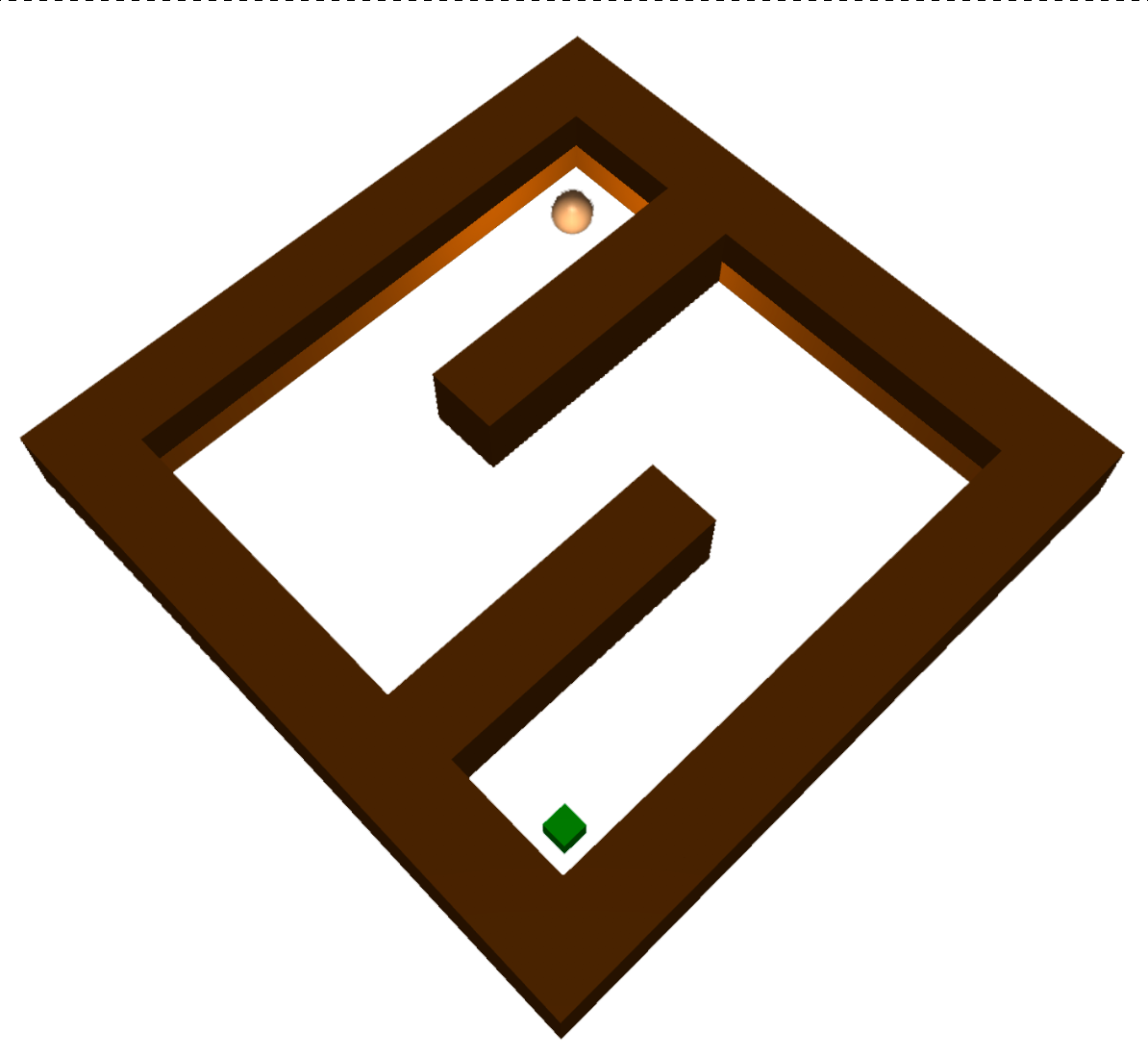}
        \caption{Zigzag Maze}
        \label{fig:s-maze-pic}
    \end{subfigure}
    \begin{subfigure}[b]{0.13\textwidth}
        \centering
        \includegraphics[width=\textwidth]{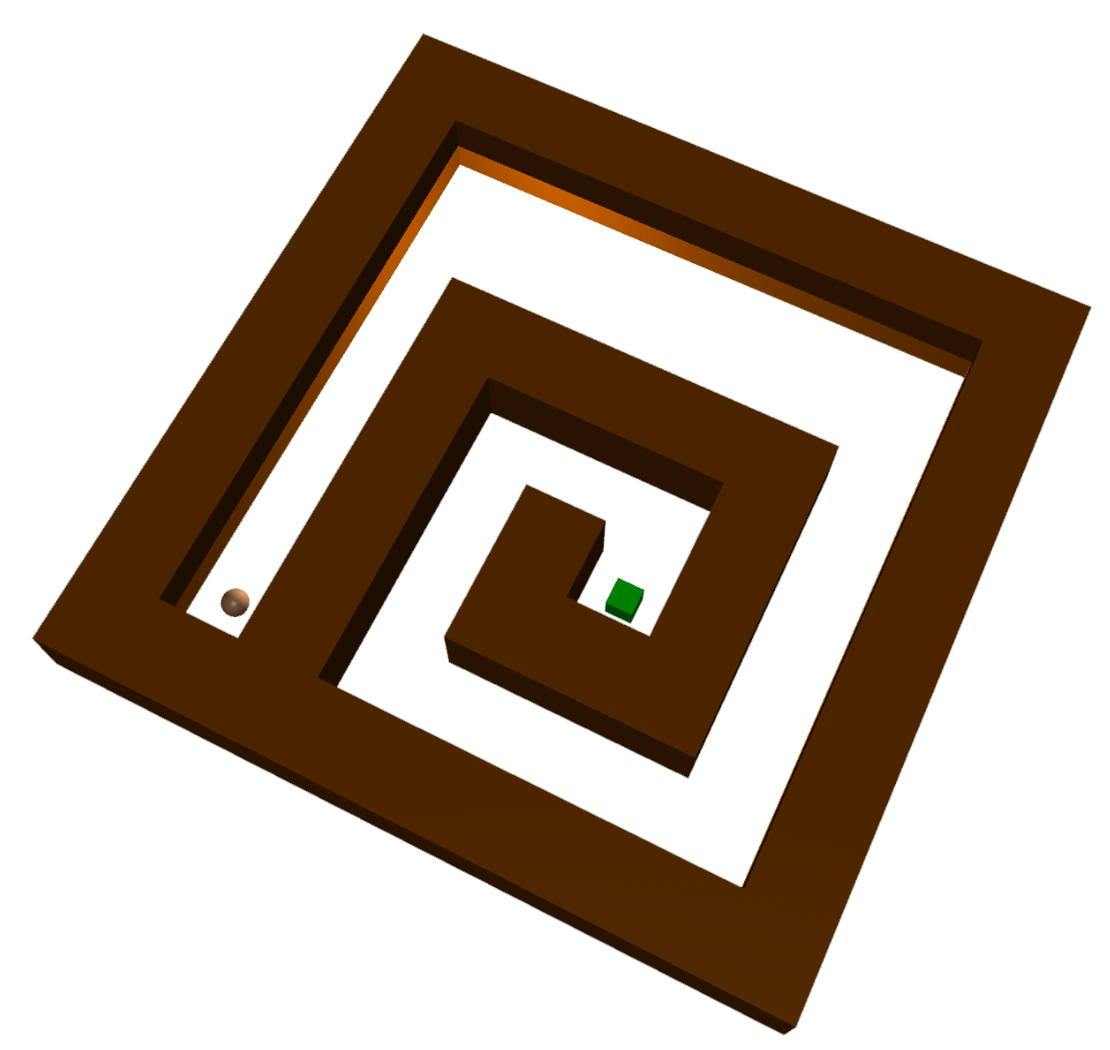}
        \caption{Spiral Maze}
        \label{fig:spiral-maze-pic}
    \end{subfigure}
    \begin{subfigure}[b]{0.13\textwidth}
        \centering
        \includegraphics[width=\textwidth]{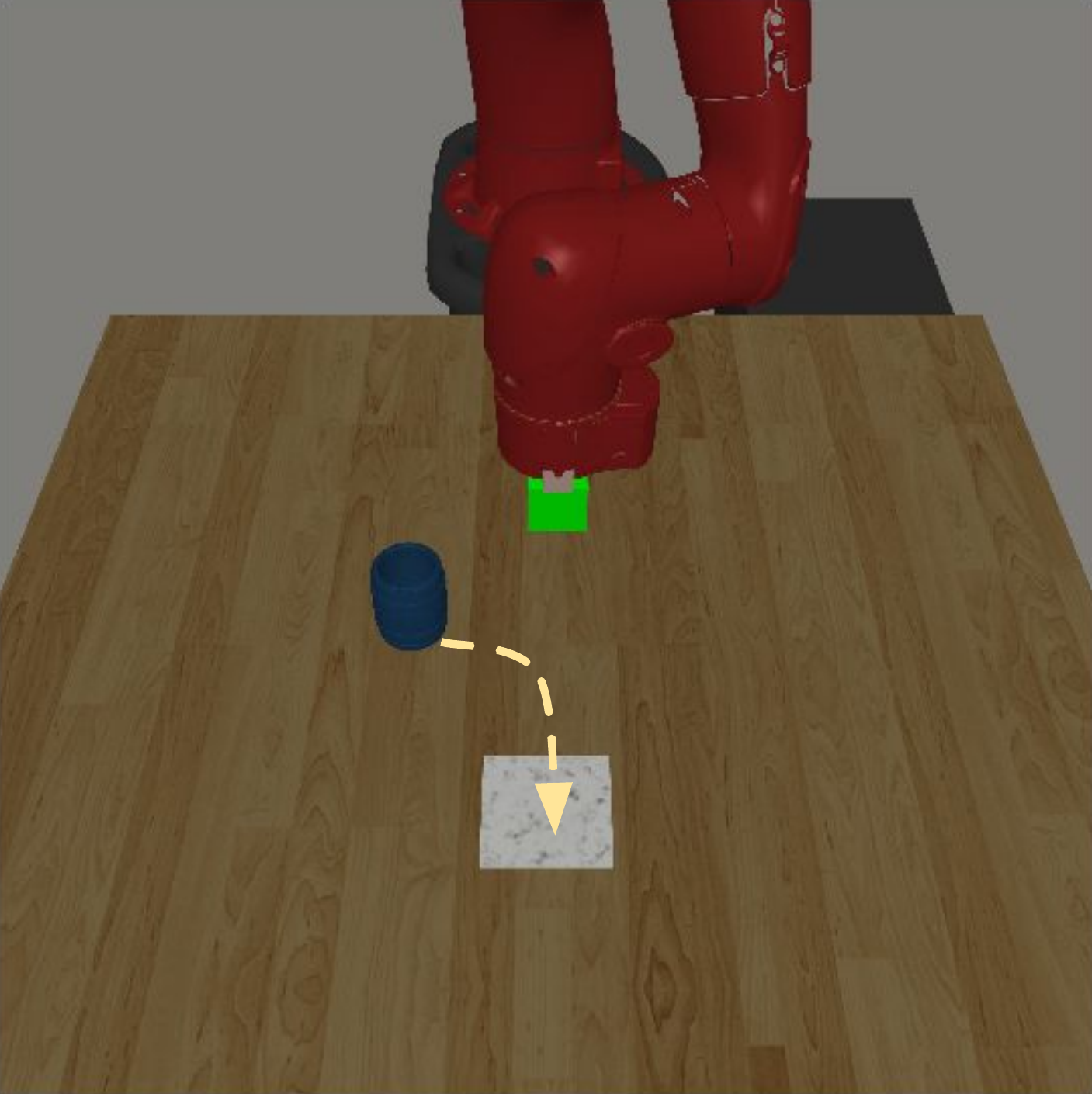}
        \caption{Sawyer Push}
        \label{fig:push-pic}
    \end{subfigure}
    \begin{subfigure}[b]{0.13\textwidth}
        \centering
        \includegraphics[width=\textwidth]{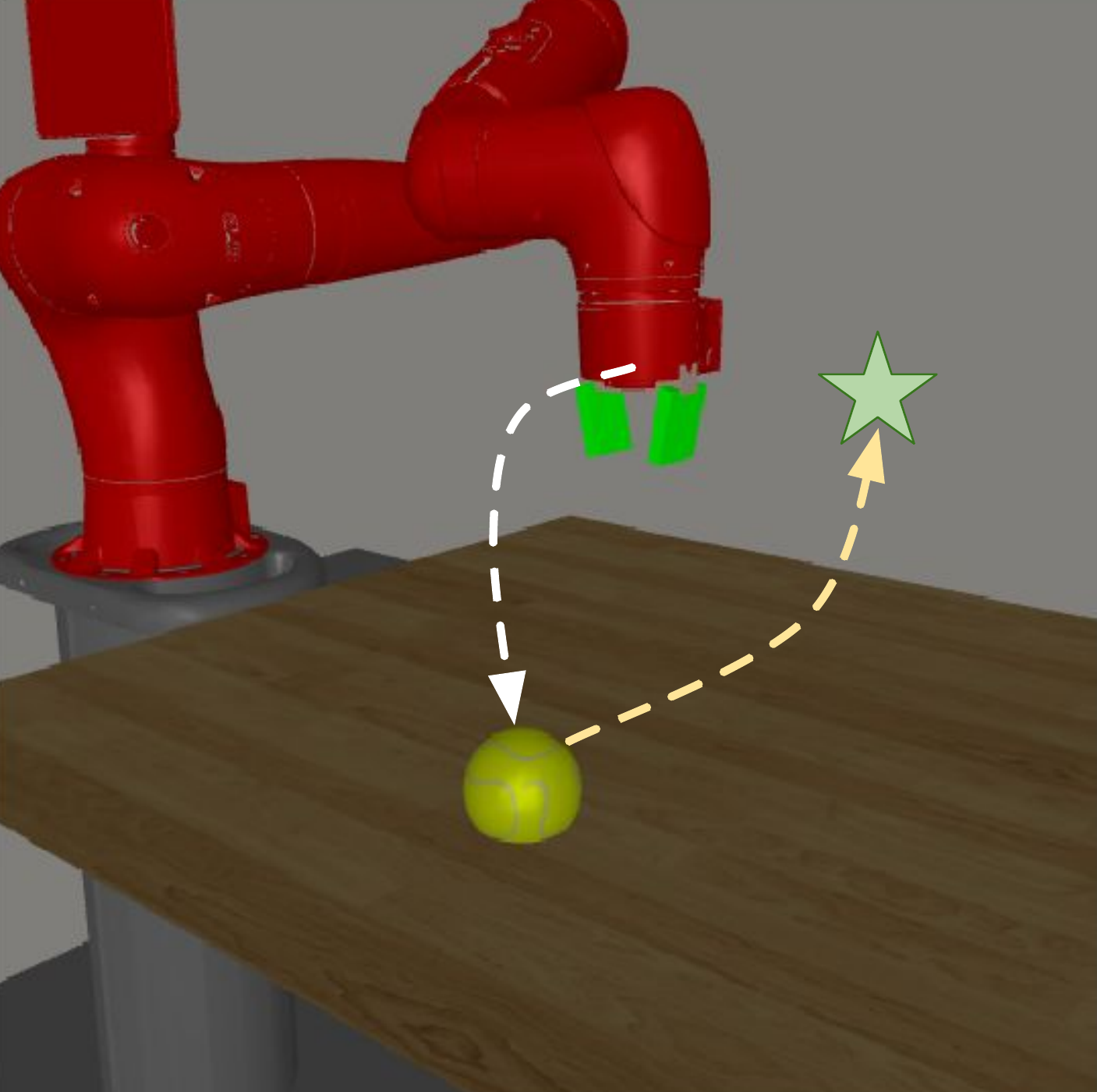}
        \caption{Sawyer Pick}
        \label{fig:pnp-pic}
    \end{subfigure}
    \begin{subfigure}[b]{0.13\textwidth}
        \centering
        \includegraphics[width=\textwidth]{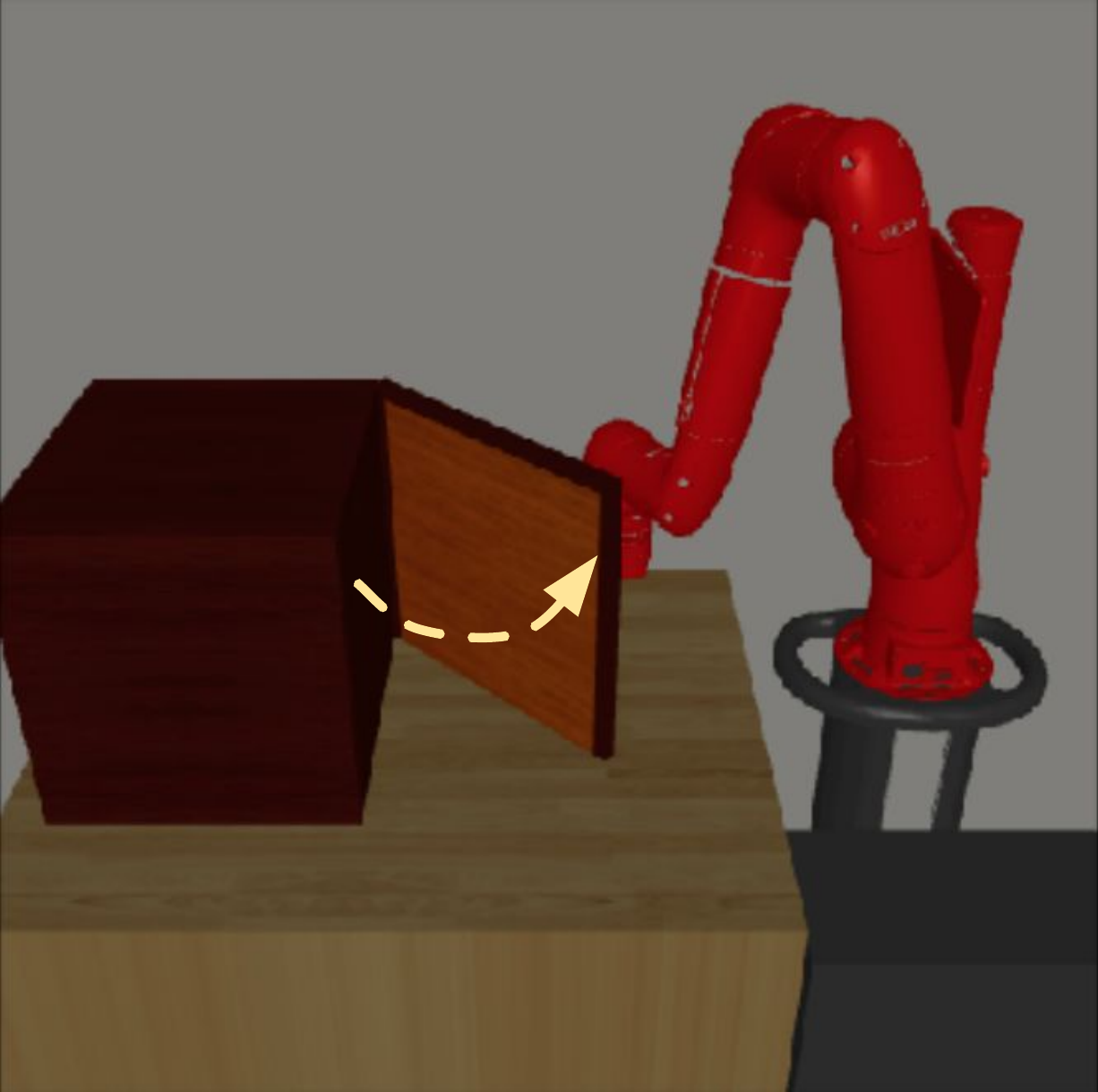}
        \caption{Sawyer Door}
        \label{fig:door-pic}
    \end{subfigure}
    \begin{subfigure}[b]{0.13\textwidth}
        \centering
        \includegraphics[width=\textwidth]{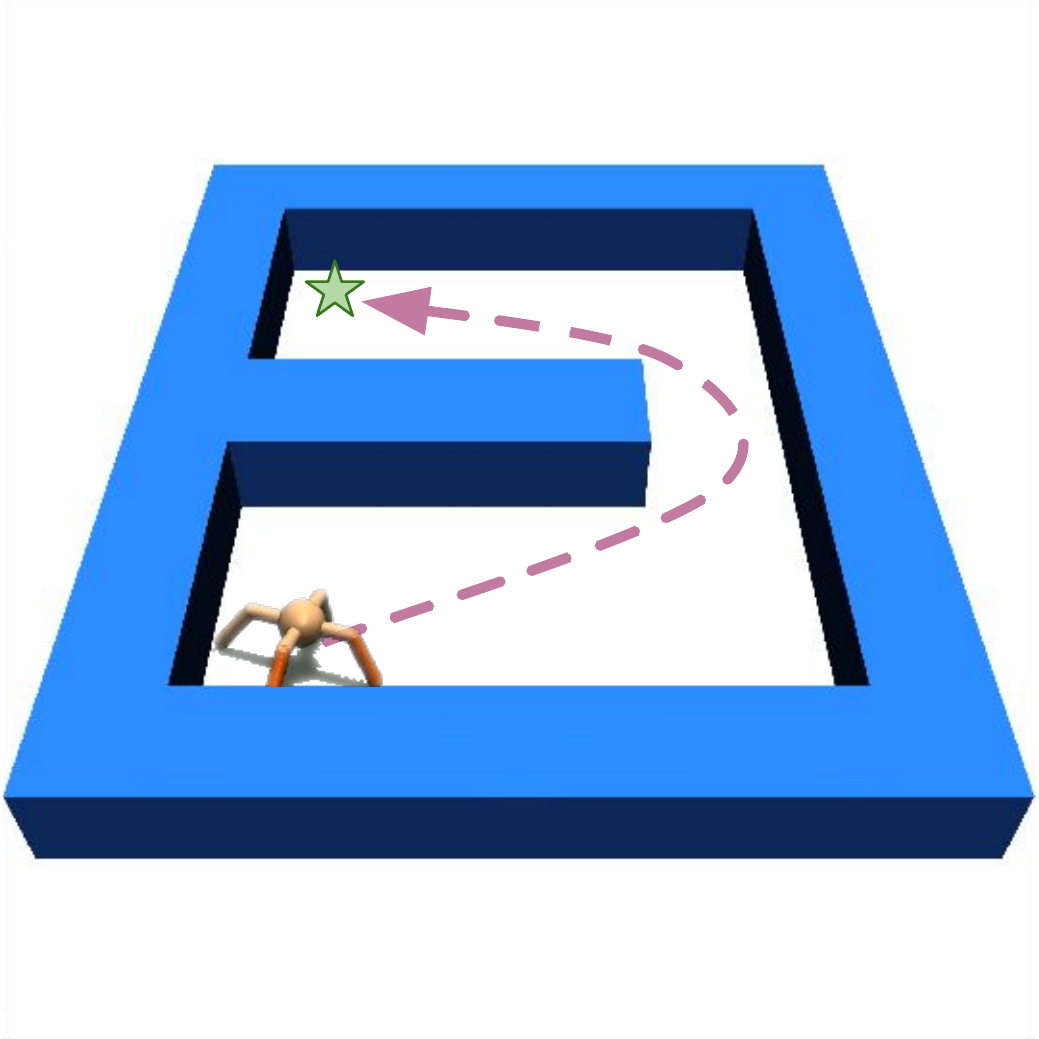}
        \caption{Locomotion}
        \label{fig:ant-pic}
    \end{subfigure}
        \begin{subfigure}[b]{0.13\textwidth}
        \centering
        \includegraphics[width=\textwidth]{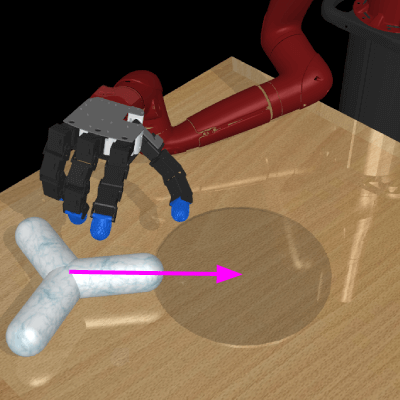}
        \caption{Dex. Hand}
        \label{fig:hand-pic}
    \end{subfigure}
    \caption{\footnotesize{We evaluate on two mazes, three robotic arm manipulation tasks, one locomotion task and one dexterous manipulation task: (a) the agent must navigate around an S-shaped corridor, (b) the agent must navigate a spiral corridor, (c) the robot must push a puck to location, (d) the robot must raise a randomly placed tennis ball to location, (e) the robot must open the door a specified angle.  (f) the quadruped ant must navigate the maze to a particular location (g) the dexterous robotic hand must reposition an object on the table.}}
        \label{fig:envs}
    \vspace{-0.3cm}
\end{figure*}

\subsection{Applying Meta-NML to Success Classification}

\begin{figure}[!h]
\vspace{-0.2cm}
\begin{minipage}{\linewidth}
    \begin{algorithm}[H]
      	\caption{\METHOD: Meta-learning Uncertainty-aware Rewards for Automated Outcome-driven RL}
      	\label{alg:baycrl-alg}
      	\begin{algorithmic}[1]
      	\STATE User provides success examples $\mathcal{S}_{+}$
      	 \STATE Initialize policy $\pi$, replay buffer $\mathcal{S}_{-}$, and reward classifier parameters $\theta_{\mathcal{R}}$
      	\FOR{iteration $i=1, 2, ...$}
      	    \STATE Add on-policy samples to $\mathcal{S}_{-}$ by executing $\pi$.
            \IF {iteration $i \mod k$ == 0}
                \STATE Sample $n_\text{train}$ states from $\mathcal{S}_{-}$ to create $2n_\text{train}$ meta-training tasks
                \STATE Sample $n_\text{test}$ total test points equally from $\mathcal{S}_{+}$ (label 1) and $\mathcal{S}_{-}$ (label 0)
                \STATE Meta-train $\theta_{\mathcal{R}}$ via meta-NML using Equation~\ref{eqn:meta-nml-training}
            \ENDIF
            \STATE Assign state rewards via Equation~\ref{eqn:CNML-RL}
            \STATE Train $\pi$ with RL algorithm
      	\ENDFOR
      	\end{algorithmic}
    \end{algorithm}
\end{minipage}
\vspace{-0.3cm}
\end{figure}

We apply the meta-NML algorithm described previously to learning uncertainty-aware success classifiers for providing rewards for RL in our proposed algorithm, which we call \METHOD. Similarly to \citet{vice18fu}, we can train our classifier by first constructing a dataset $\mathcal{D}$ for binary classification, using success examples as positives, and on-policy samples as negatives, balancing the number of sampled positives and negatives in the dataset. Given this dataset, the classifier parameters $\theta_{\mathcal{R}}$ can be trained via meta-NML as described in Equation~\ref{eqn:meta-nml-training}. The classifier can then be used to directly and quickly assign rewards to a state $s$ according to its probabilities $r(s) = p_{\text{meta-NML}}(e=1|s)$, and perform standard reinforcement learning, as noted in Algorithm~\ref{alg:baycrl-alg}. Further details are in Appendix A.2. 

\section{Experimental Evaluation}
\label{sec:experiments}

In our experimental evaluation we aim to answer the following questions: (1) Can \METHOD make effective use of successful outcome examples to solve challenging exploration tasks? (2) Does \METHOD scale to dynamically complex tasks? (3) What are the impacts of different design decisions on the effectiveness of \METHOD?

\begin{figure*}[!h]
\vspace{-0.1cm}
    \centering
    \includegraphics[width=0.8\textwidth]{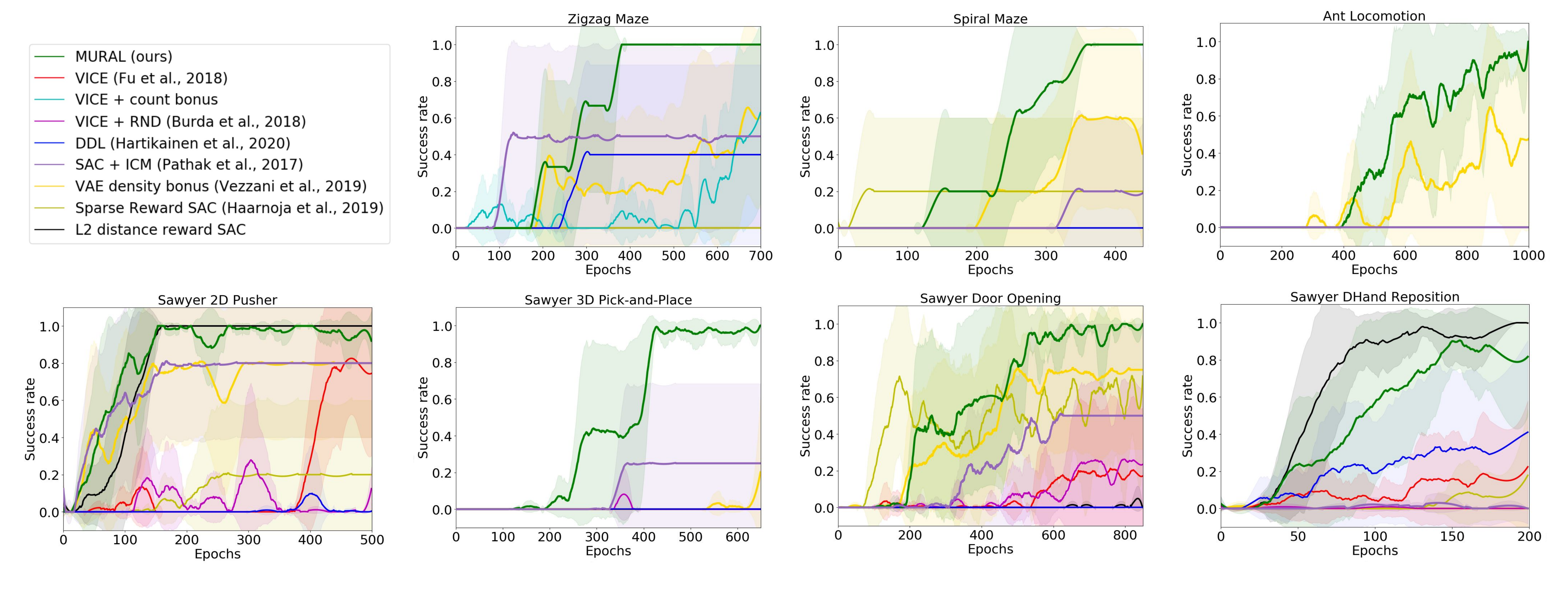}
    \vspace{-0.2cm}
    \caption{\footnotesize{\METHOD outperforms prior goal-reaching and exploration methods on all our evaluation environments, including ones with high-dimensional state and action spaces. \METHOD also performs comparably to or better than a heuristically shaped hand-designed reward that uses Euclidean distance (black line), demonstrating that designing a well-shaped reward is not trivial in these domains. Shading indicates a standard deviation across 5 seeds. For details on the success metrics used, see Appendix A.4.
    }}
    \label{fig:comparisons}
\end{figure*}

Further details, videos, and code can be found at \mbox{\url{https://sites.google.com/view/mural-rl}}

\subsection{Experimental Setup}

We first evaluate our method on maze navigation problems, which require avoiding several local optima. Then, we consider three robotic manipulation tasks that were previously covered in ~\citet{viceraq} with a Sawyer robot arm: door opening, tabletop object pushing, and 3D object picking. We also evaluate on a previously considered locomotion task ~\cite{pong2020skew} with a quadruped ant navigating to a target, as well as a dexterous manipulation problem with a robot repositioning an object with a multi-fingered hand. In the hand manipulation experiments, the classifier is provided with access to only the object position, while in the other tasks the classifier is provided the entire Markovian state. As we show in our results, exploration in these environments is challenging, and using na\"{i}vely chosen reward shaping often does not solve the problem at hand. 

We compare with a number of prior algorithms. To provide a comparison with a previous method that uses standard success classifiers, we include the VICE algorithm~\citep{vice18fu}. Note that this algorithm is quite related to \METHOD, but it uses a standard maximum likelihood classifier rather than a classifier trained with CNML and meta-learning. We also include a comparison with DDL, a technique for learning dynamical distances~\citep{ddl}. We additionally include comparisons to algorithms for task-agnostic exploration to show that \METHOD performs more directed exploration and reward shaping. To provide a direct comparison, we use the same VICE method for training classifiers, but combine it with novelty-based exploration based on random network distillation~\citep{rnd} for the robotic manipulation tasks, and oracle inverse count bonuses for maze navigation. We also compare to prior task-agnostic exploration techniques which use intrinsic curiosity ~\citep{pathak2017curiosity} and density estimates ~\citep{vezzanivae}. Finally, to demonstrate the importance of shaped rewards, we compare to running Soft Actor-Critic~\citep{sac} with two na\"{i}ve reward functions: a sparse reward, and a heuristic reward which uses L2 distance. More details are included in Appendix A.4 and A.6.

\subsection{Comparisons with Prior Algorithms}
We compare with prior algorithms on the domains described above. As we can see in Fig~\ref{fig:comparisons}, \METHOD is able to very quickly learn how to solve these challenging exploration tasks, often reaching better asymptotic performance than most prior methods, and doing so more efficiently than VICE~\citep{vice18fu} or DDL~\citep{ddl}. This suggests that \METHOD is able to provide directed reward shaping and exploration that is substantially better than standard classifier-based methods. We provide a more detailed analysis of the shaping behavior of the learned reward in Section~\ref{sec:exp-viz}.

To isolate whether the benefits purely come from exploration or also from task-aware reward shaping, we compare with methods that only perform task-agnostic exploration. From these comparisons, it is clear that \METHOD significantly outperforms methods that only use novelty-seeking exploration. We also compare to a heuristically-designed reward function based on Euclidean distance. \METHOD generally outperforms simple manual shaping in terms of sample complexity and asymptotic performance, indicating that the learned shaping is non-trivial and adapted to the task. Of course, with sufficient domain knowledge, it is likely that this would improve. In addition, we find that \METHOD scales up to tasks with challenging exploration in higher dimensional state and action spaces such as quadruped locomotion and dexterous manipulation, as seen in Fig~\ref{fig:comparisons}.

\subsection{Ablations}
We first evaluate the importance of meta-learning for estimating the CNML distribution. In Figure~\ref{fig:ablations}, we see that na\"{i}vely estimating the CNML distribution by taking a single gradient step and following the same process as in our method, but without any meta-training, results in much worse performance.
Second, we analyze whether the exploration behavior incentivized by \METHOD is actually directed and task-aware or if it simply approximates count-based exploration.
To that end, we modify the training procedure so that the dataset $\mathcal{D}$ consists of only the on-policy negatives, and add the inferred reward from the meta-NML classifier to the reward obtained by a standard maximum likelihood classifier (similarly to the VICE+RND baseline). We see that this performs poorly, showing that the \METHOD classifier is doing more than just performing count-based exploration, and benefits from better reward shaping due to the success examples. Further ablations are available in Appendix A.5.

\begin{wrapfigure}{r}{0.45\linewidth}
    \centering
    \includegraphics[width=\linewidth]{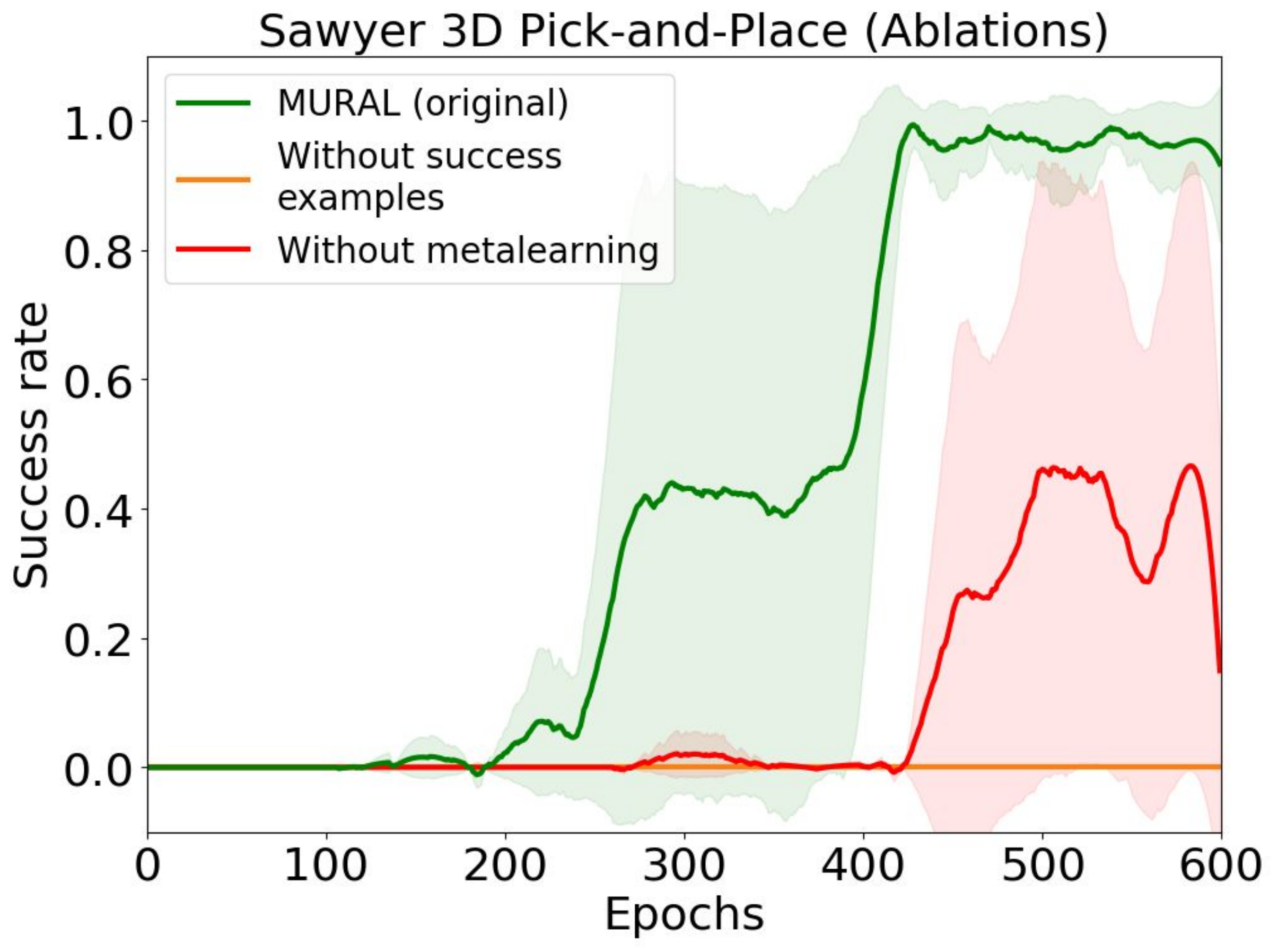}
    \caption{{\footnotesize Ablative analysis of \METHOD. The amortization from meta-learning and access to positive examples are both important components for performance.}}
    \label{fig:ablations}
    \vspace{-0.4cm}
\end{wrapfigure}

\subsection{Analysis of \METHOD}
\label{sec:exp-viz}

\textbf{\METHOD and reward shaping.} To better understand how \METHOD provides reward shaping, we visualize the rewards for various slices along the $z$ axis on the Sawyer Pick-and-Place task, an environment which presents a significant exploration challenge. In Fig~\ref{fig:shaping} we see that the \METHOD rewards clearly correlate with the distance to the mean object position in successful outcomes, shown as a white star, thus guiding the robot to raise the ball to the desired location even if it has never reached this before. In contrast, the maximum likelihood classifier has a sharp, poorly-shaped decision boundary.

\begin{figure}[!h]
    \centering
    \includegraphics[width=0.95\linewidth]{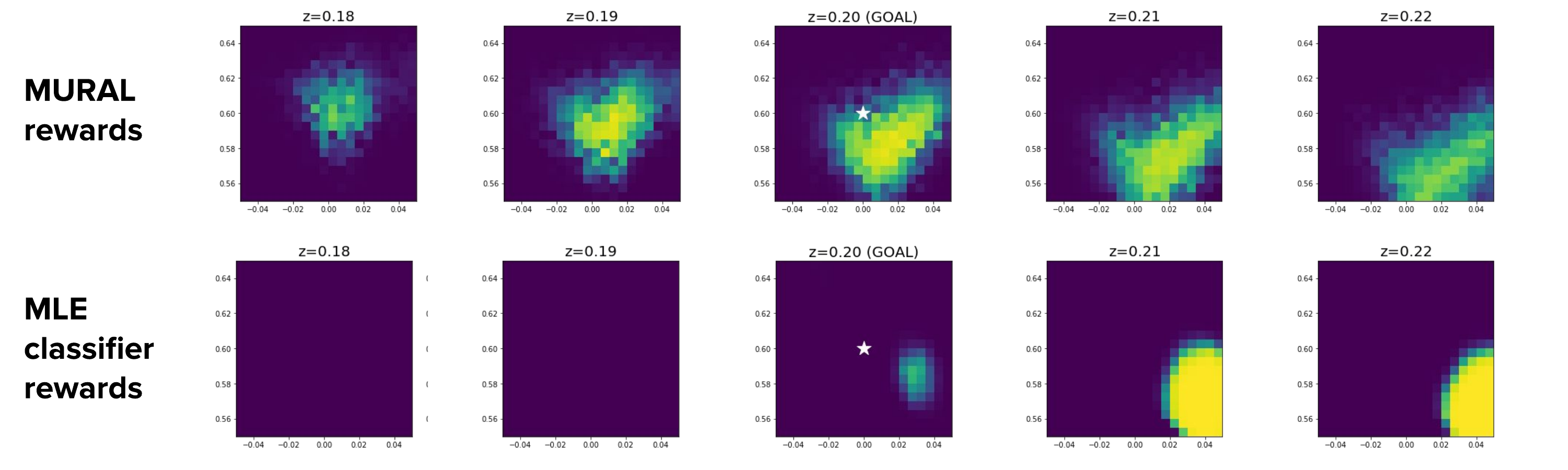}
    \vspace{-0.3cm}
    \caption{\footnotesize{Visualization of reward shaping for 3D Pick-and-Place at various z values (heights). \METHOD learns rewards that provide a smooth slope toward the successful outcomes, while the MLE classifier learns a sharp and poorly shaped decision boundary.}}
    \label{fig:shaping}
\end{figure}
\textbf{\METHOD and exploration.} Next, to illustrate the connection between \METHOD and exploration, we compare the states visited by \METHOD and by VICE \citep{vice18fu} in Figure~\ref{fig:visitations}. We see that \METHOD naturally incentivizes the agent to visit novel states, allowing it to navigate around local minima. In contrast, VICE learns a misleading reward function that prioritizes closeness to the success outcomes in $xy$ space, causing the agent to get stuck.

Interestingly, despite incentivizing exploration, \METHOD does not simply visit all possible states; at convergence, it has only covered around 70\% of the state space. In fact, in Figure~\ref{fig:visitations}, \METHOD prioritizes states that bring it closer to the success outcomes and ignores ones that don't, making use of the positive examples provided to it. This suggests that \METHOD benefits from both novelty-seeking behavior and effective reward shaping.  

\begin{figure}[!h]
    \centering
    \includegraphics[width=0.9\linewidth]{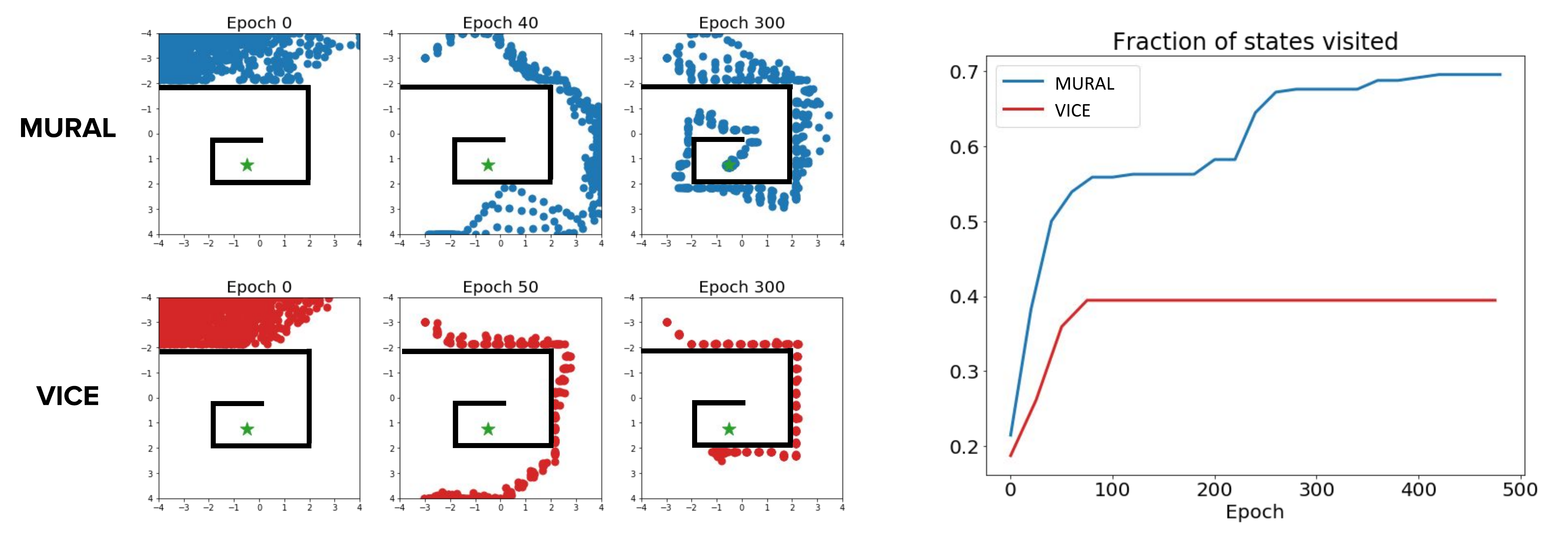}
    \caption{\footnotesize{Plots of visitations and state coverage over time for \METHOD vs. VICE. \METHOD explores a significantly larger portion of the state space and is able to avoid local optima.}}
    \label{fig:visitations}
    \vspace{-0.3cm}
\end{figure}
\section{Discussion}
In this work, we consider a subclass of RL problems where examples of successful outcomes specify the task. We analyze how solutions via standard success classifiers suffer from shortcomings, and training classifiers via CNML allows for better exploration to solve challenging problems. To make learning tractable, we propose a novel meta-learning approach to amortize the CNML process. While this work has shown the effectiveness of Bayesian classifiers for reward inference for tasks in simulation, it would be interesting to scale this solution to real world problems.

\section{Acknowledgements}
The authors would like to thank Aviral Kumar, Justin Fu, Sham Kakade, Aldo Pacchiano, Pieter Abbeel, Avi Singh, Benjamin Eysenbach, Ignasi Clavera for valuable discussion and feedback on early drafts of this work. This research was supported by the Office of Naval Research and the National Science Foundation.

\bibliography{references}
\bibliographystyle{icml2021}

\clearpage
\appendix

\section{Appendix}

\subsection{Proof of Theorem 1 connecting NML and inverse counts}\label{appsec:proof}
We provide the proof of Theorem 1 here for completeness.
\begin{theorem}
Suppose we are estimating success probabilities $p(e=1\vert s)$ in the tabular setting, where we have a separate parameter independently for each state.
Let $N(s)$ denote the number of times state $s$ has been visited by the policy, and let $G(s)$ be the number of occurrences of state $s$ in the successful outcomes.
Then the CNML probability $p_{\text{CNML}}(e=1\vert s)$ is equal to $\frac{G(s)+1}{N(s)+G(s)+2}$. For states that are never observed to be successful, we then recover inverse counts $\frac{1}{N(s) + 2}.$
\end{theorem}
\begin{proof}
In the fully tabular setting, our MLE estimates for $p(O\vert s)$ are simply given by finding the best parameter $p_s$ for each state. The proof then proceeds by simple calculation.

For a state with $n = N(s)$ negative occurrences and $g = G(s)$ positive occurrences, the MLE estimate is simply given by $\frac{g}{n+g}$.

Now for evaluating CNML, we consider appending another instance for each class.
The new parameter after appending a negative example is then $\frac{g}{n+g+1}$, which then assigns probability $\frac{n+1}{n+g+1}$ to the negative class.
Similarly, after appending a positive example, the new parameter is $\frac{g+1}{n+g+1}$, so we try to assign probability $\frac{g+1}{n+g+1}$ to the positive class.
Normalizing, we have 
\begin{align}
    p_{\text{CNML}}(O=1\vert s) = \frac{g+1}{n+g+2}.
\end{align}

When considering states that have only been visited on-policy, and are not included in the set of successful outcomes, then the likelihood reduces to 

\begin{align}
    p_{\text{CNML}}(O=1\vert s) = \frac{1}{n+2}.
\end{align}
\end{proof}

\subsection{Detailed Description of Meta-NML} \label{sec:detailed-pseudocode-meta-nml}
We provide a detailed description of the meta-NML algorithm described in Section~\ref{sec:meta-nml}, and the details of the practical algorithm.

\begin{figure}[!h]
    \centering
    \includegraphics[width=\linewidth]{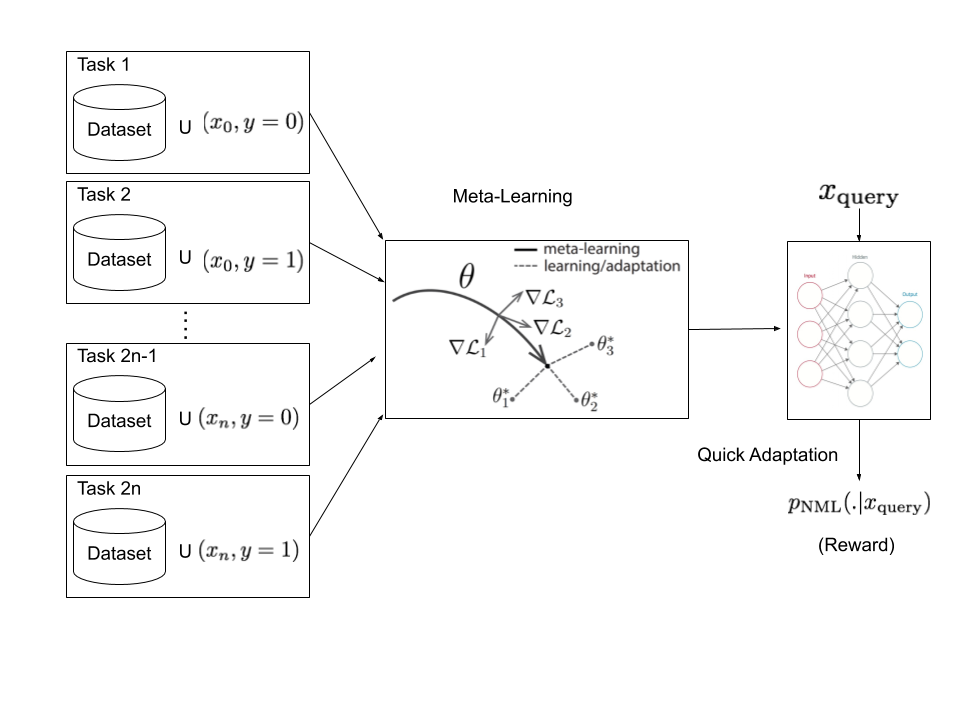}
    \caption{Figure illustrating the meta-training procedure for meta-NML.}
    \label{fig:meta-nml}
    \vspace{-0.2pt}
\end{figure}

Given a dataset $\mathcal{D} = \{(x_0, y_0), (x_1, y_1),..,(x_n, y_n)\}$, the meta-NML procedure proceeds by first constructing $k*n$ tasks from these data points, for a $k$ shot classification problem. We will keep $k = 2$ for simplicity in this description, in accordance with the setup of binary success classifiers in RL. Each task $\tau_i$ is constructed by augmenting the dataset with a negative label $\mathcal{D} \cup (x_i, y=0)$ or a positive label $\mathcal{D} \cup (x_i, y=1)$. Now that each task consists of solving the maximum likelihood problem for its augmented dataset, we can directly apply standard meta-learning algorithms to this setting. Building off the ideas in MAML~\citep{finn17icml}, we can then train a set of model parameters $\theta$ such that after a single step of gradient descent it can quickly adapt to the optimal solution for the MLE problem on any of the augmented datasets. This is more formally written as 

\begin{equation}
\max_{\theta} \hspace{0.1cm} \mathbb{E}_{\tau \sim \mathcal{S}(\tau)}[\mathcal{L}(\tau, \theta')],\hspace{0.4cm} s.t \hspace{0.2cm} \theta' = \theta - \alpha \nabla_{\theta} \mathcal{L}(\tau, \theta)
\label{eqn:meta-training-nml-appendix}
\end{equation}

where $\mathcal{L}$ represents a standard classification loss function, $\alpha$ is the learning rate, and the distribution of tasks $p(\tau)$ is constructed as described above. For a new query point $x$, these initial parameters can then quickly be adapted to provide the CNML distribution by taking a gradient step on each augmented dataset to obtain the approximately optimal MLE solution, and normalizing these as follows: 

\begin{align*}
&p_{\text{meta-NML}}(y|x; \mathcal{D}) = \frac{p_{\theta_y}(y|x)}{\sum_{y \in \mathcal{Y}}p_{\theta_y}(y|x)} \\ &\theta_y = \theta - \alpha \nabla_\theta \mathbb{E}_{(x_i, y_i) \sim \mathcal{D} \cup (x, y)}[\mathcal{L}(x_i, y_i, \theta)]
\label{eqn:meta-testing-nml-appendix}
\end{align*}

This algorithm in principle can be optimized using any standard stochastic optimization method such as SGD, as described in ~\citet{finn17icml}, backpropagating through the inner loop gradient update. For the specific problem setting that we consider, we additionally employ some optimization tricks in order to enable learning: 

\subsubsection{Importance Weighting on Query Point}
Since only one datapoint is augmented to the training set at query time for CNML, stochastic gradient descent can ignore this datapoint with increasing dataset sizes. For example, if we train on an augmented dataset of size 2048 by cycling through it in batch sizes of 32, then only 1 in 64 batches would include the query point itself and allow the model to adapt to the proposed label, while the others would lead to noise in the optimization process, potentially worsening the model's prediction on the query point.

In order to make sure the optimization considers the query point, we include the query point and proposed label $(x_q, y)$ in every minibatch that is sampled, but downweight the loss computed on that point such that the overall objective remains unbiased. This is simply doing importance weighting, with the query point downweighted by a factor of $\lceil \frac{b - 1}{N} \rceil$ where $b$ is the desired batch size and $N$ is the total number of points in the original dataset.

To see why the optimization objective remains the same, we can consider the overall loss over the dataset. Let $f_\theta$ be our classifier, $\mathcal{L}$ be our loss function, $\mathcal{D'} = \{(x_i, y_i)\}_{i=1}^N \cup (x_q, y)$ be our augmented dataset, and $\mathcal{B}_k$ be the $k$th batch seen during training. Using standard SGD training that cycles through batches in the dataset, the overall loss on the augmented dataset would be:

$$\mathcal{L}(\mathcal{D'}) = \left( \sum \limits_{i=0}^{N} \mathcal{L}(f_\theta(x_i), y_i) \right) + \mathcal{L}(f_\theta(x_q), y)$$

If we instead included the downweighted query point in every batch, the overall loss would be:

\begin{align*}
    \mathcal{L}(\mathcal{D'}) &= \sum \limits_{k=0}^{\lceil \frac{b-1}{N} \rceil} \sum \limits_{(x_i, y_i) \in \mathcal{B}_k} \left( \mathcal{L}(f_\theta(x_i), y_i) + \frac{1}{\lceil \frac{b-1}{N} \rceil} \mathcal{L}(f_\theta(x_q), y) \right) \\
    &= \left( \sum \limits_{k=0}^{\lceil \frac{b-1}{N} \rceil} \sum \limits_{(x_i, y_i) \in \mathcal{B}_k} \mathcal{L}(f_\theta(x_i), y_i) \right) + \\ &\lceil \frac{b-1}{N} \rceil \frac{1}{\lceil \frac{b-1}{N} \rceil} \mathcal{L}(f_\theta(x_q), y) \\
    &= \left( \sum \limits_{i=0}^{N} \mathcal{L}(f_\theta(x_i), y_i) \right) + \mathcal{L}(f_\theta(x_q), y)
\end{align*}

which is the same objective as before.

This trick has the effect of still optimizing the same maximum likelihood problem required by CNML, but significantly reducing the variance of the query point predictions as we take additional gradient steps at query time. As a concrete example, consider querying a meta-CNML classifier on the input shown in Figure~\ref{fig:importance-weighting-plots}. If we adapt to the augmented dataset without including the query point in every batch (i.e. without importance weighting), we see that the query point loss is significantly more unstable, requiring us to take more gradient steps to converge. 

\begin{figure}[!h]
    \centering
    \includegraphics[width=0.95\linewidth]{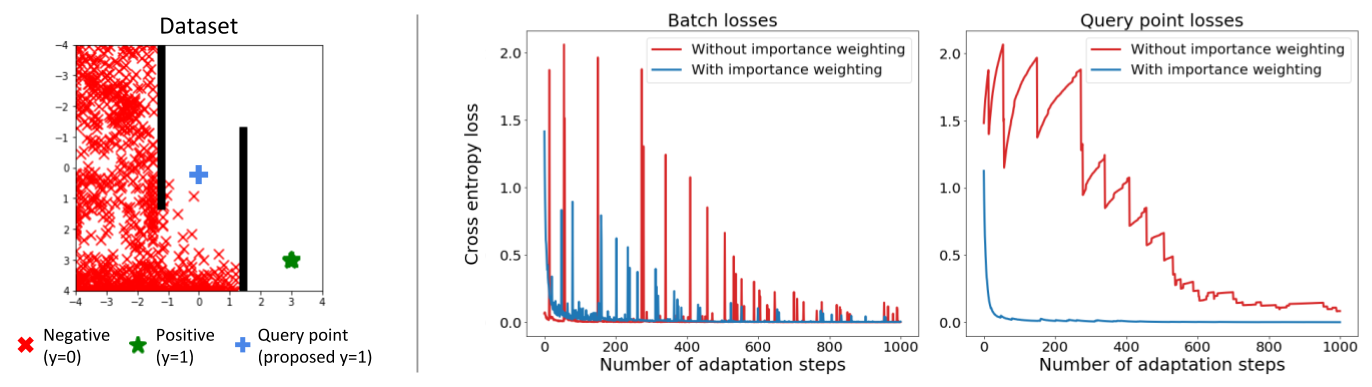}
    \caption{Comparison of adapting to a query point (pictured on left with the original dataset) at test time for CNML with and without importance weighting. The version without importance weighting is more unstable both in terms of overall batch loss and the individual query point loss, and thus takes longer to converge. The spikes in the red lines occur when that particular batch happens to include the query point, since that point's proposed label ($y=1$) is different than those of nearby points ($y=0$). The version with importance weighting does not suffer from this problem because it accounts for the query point in each gradient step, while keeping the optimization objective the same.}
    \label{fig:importance-weighting-plots}
\end{figure}

\subsubsection{Kernel Weighted Training Loss}
The augmented dataset consists of points from the original dataset $\mathcal{D}$ and one augmented point $(x_q, y)$. Given that we mostly care about having the proper likelihood on the query point, with an imperfect optimization process, the meta-training can yield solutions that are not very accurately representing true likelihoods on the query point. To counter this, we introduce a kernel weighting into the loss function in Equation ~\ref{eqn:meta-training-nml-appendix} during meta-training and subsequently meta-testing. The kernel weighting modifies the training loss function as: 

\begin{align*}
\max_{\theta} \hspace{0.1cm} &\mathbb{E}_{\tau \sim \mathcal{S}(\tau)}[\mathbb{E}_{(x, y) \sim \tau}\mathcal{K}(x, x_{\tau})\mathcal{L}(x, y, \theta')] \\ &s.t \hspace{0.2cm} \theta' = \theta - \alpha \nabla_{\theta} \mathbb{E}_{(x, y) \sim \tau} \mathcal{K}(x, x_{\tau})\mathcal{L}(x, y, \theta)
\end{align*}

where $x_{\tau}$ is the query point for task $\tau$ and $\mathcal{K}$ is a choice of kernel. We typically choose exponential kernels centered around $x_{\tau}$. Intuitively, this allows the meta-optimization to mainly consider the datapoints that are copies of the query point in the dataset, or are similar to the query point, and ensures that they have the correct likelihoods, instead of receiving interfering gradient signals from the many other points in the dataset. To make hyperparameter selection intuitive, we designate the strength of the exponential kernel by a parameter $\lambda_{dist}$, which is the Euclidean distance away from the query point at which the weight becomes 0.1. Formally, the weight of a point $x$ in the loss function for query point $x_\tau$ is computed as:

\begin{equation}
    K(x, x_\tau) = \exp{\{-\frac{2.3}{\lambda_{dist}}||x - x_\tau||_2}\}
\end{equation}

\subsubsection{Meta-Training at Fixed Intervals}
While in principle meta-NML would retrain with every new datapoint, in practice we retrain meta-NML once every $k$ epochs. (In all of our experiments we set $k = 1$, but we could optionally increase $k$ if we do not expect the meta-task distribution to change much between epochs.) We warm-start the meta-learner parameters from the previous iteration of meta-learning, so every instance of meta-training only requires a few steps. We find that this periodic training is a reasonable enough approximation, as evidenced by the strong performance of \METHOD in our experimental results in Section~\ref{sec:experiments}.

\subsection{Meta-NML Visualizations}

\subsubsection{Meta-NML with Additional Gradient Steps}
\label{sec:meta-nml-grad-steps}
Below, we show a more detailed visualization of meta-NML outputs on data from the Zigzag Maze task, and how these outputs change with additional gradient steps. For comparison, we also include the idealized NML rewards, which come from a discrete count-based classifier.

Meta-NML is able to resemble the ideal NML rewards fairly well with just 1 gradient step, providing both an approximation of a count-based exploration bonus and better shaping towards the goal due to generalization. By taking additional gradient steps, meta-NML can get arbitrarily close to the true NML outputs, which themselves correspond to inverse counts of $\frac{1}{n+2}$ as explained in Theorem 4.1. While this would give us more accurate NML estimates, in practice we found that taking one gradient step was sufficient to achieve good performance on our RL tasks.

\begin{figure}[!h]
    \centering
    \includegraphics[width=0.9\linewidth]{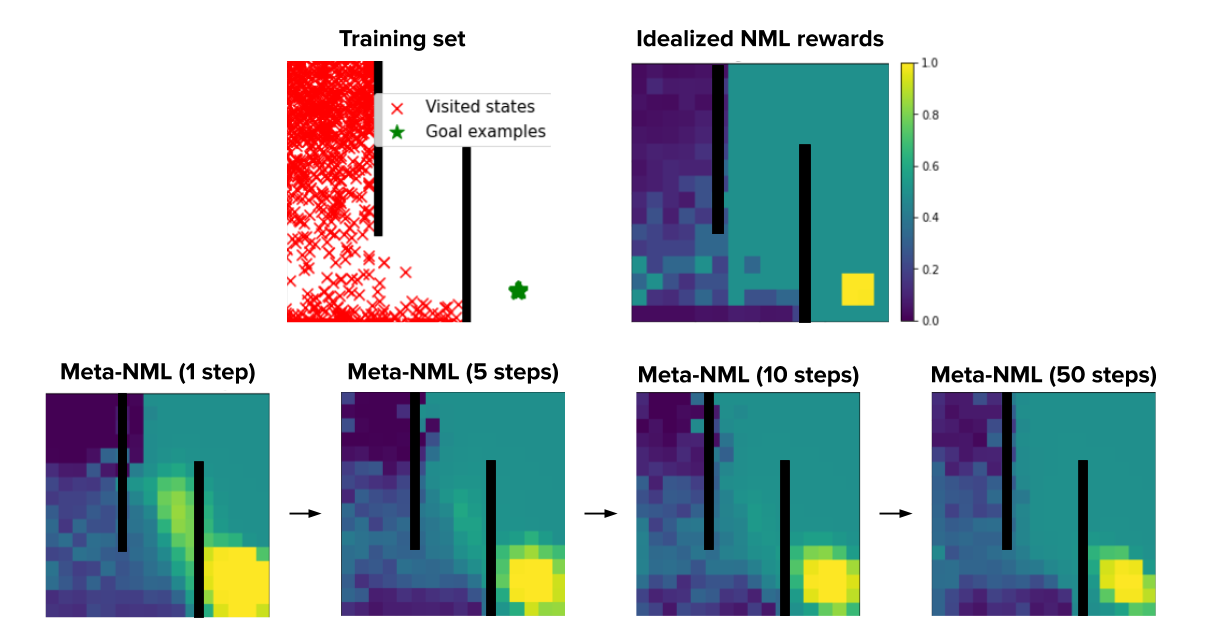}
    \caption{Comparison of idealized (discrete) NML and meta-NML rewards on data from the Zigzag Maze Task. Meta-NML approximates NML reasonably well with just one gradient step at test time, and converges to the true values with additional steps.}
    \label{fig:nml-comparisons-discrete}
\end{figure}

\begin{figure}[!h]
    \centering
    \includegraphics[width=0.7\linewidth]{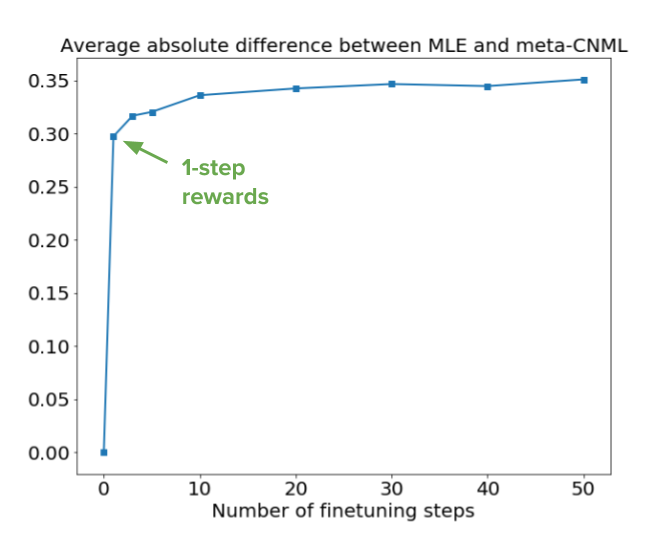}
    \caption{Average absolute difference between MLE and meta-NML goal probabilities across the entire maze state space from Figure~\ref{fig:nml-comparisons-discrete} above. We see that meta-NML learns a model initialization whose parameters can change significantly in a small number of gradient steps. Additionally, most of this change comes from the first gradient step (indicated by the green arrow), which justifies our choice to use only a single gradient step when evaluating meta-NML probabilities for \METHOD.}
    \label{fig:mle-cnml-diff}
\end{figure}

\subsubsection{Comparison of Reward Classifiers}
In Fig~\ref{fig:classifier-reward-comparisons}, we show the comparison between different types of reward. classifiers in the 2D maze navigation problem.
\begin{figure}[!h]
    \centering
    \includegraphics[width=\linewidth]{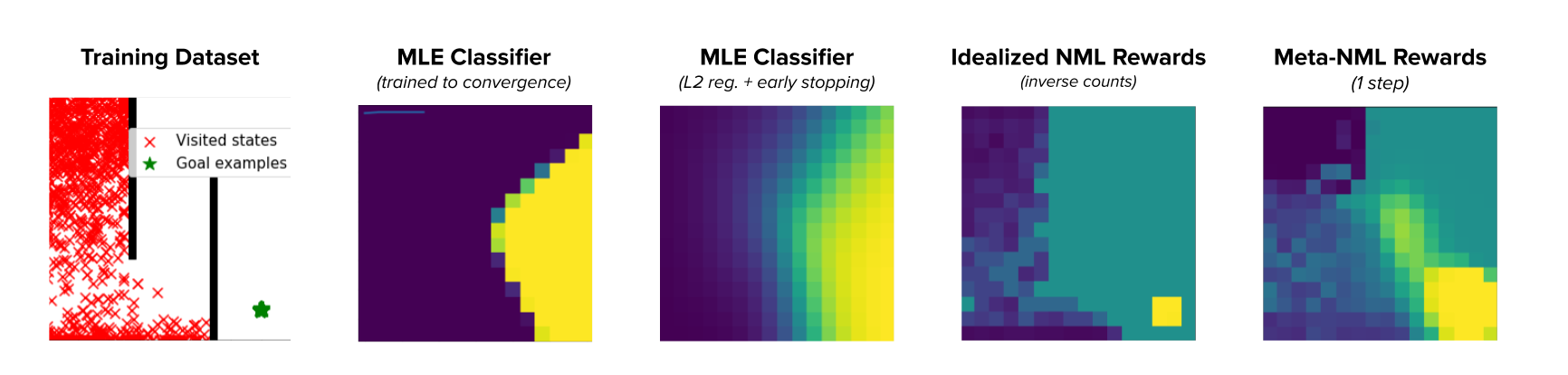}
    \caption{A comparison of the rewards given by various classifier training schemes on the 2D Zigzag maze. From left to right: (1) An MLE classifier when trained to convergence reduces to an uninformative sparse reward; (2) An MLE classifier trained with regularization and early stopping has smoother contours, but does not accurately identify the goal; (3) The idealized NML rewards correspond to inverse counts, thus providing a natural exploration objective in the absence of generalization; (4) The meta-NML rewards approximate the idealized rewards well in visited regions, while also benefitting from better shaping towards the goal due to generalization.}
    \label{fig:classifier-reward-comparisons}
\end{figure}

\subsubsection{Runtime Comparisons}
\label{sec:full-runtimes}
We provide the runtimes for feedforward inference, naive CNML, and meta-NML on each of our evaluation domains. We list both the runtimes for evaluating a single input (Table~\ref{tab:meta-nml-runtimes}), and for completing a full epoch of training during RL (Table~\ref{tab:meta-nml-runtimes-per-epoch}).

These benchmarks were performed on an NVIDIA Titan X Pascal GPU. Per-input runtimes are averaged across 100 samples, and per-epoch runtimes are averaged across 10 epochs.

\begin{table*}[htb]
    \centering
    \begin{tabular}{|l|l|l|l|} \hline
    & \textbf{Feedforward}              & \textbf{Meta-NML}               & \textbf{Naive CNML} \\ \hline
    \textbf{Mazes (zigzag, spiral)} & 0.0004s              & 0.0090s         & 15.19s              \\ \hline
    \textbf{Sawyer 2D Pusher} & 0.0004s & 0.0092s & 20.64s \\ \hline
    \textbf{Sawyer Door} & 0.0004s & 0.0094s & 20.68s \\ \hline
    \textbf{Sawyer 3D Pick} & 0.0005s & 0.0089s & 20.68s \\ \hline
    \textbf{Ant Locomotion} & 0.0004s & 0.0083s & 17.26s \\ \hline
    \textbf{Dexterous Manipulation} & 0.0004s & 0.0081s & 17.58s \\ \hline
    \end{tabular}
    \caption{Runtimes for evaluating a single input point using feedforward, meta-NML, and naive CNML classifiers. Meta-NML provides anywhere between a 1600x and 2300x speedup compared to naive CNML, which is crucial to making our NML-based reward classifier scheme feasible on RL problems.}
    \label{tab:meta-nml-runtimes}
\end{table*}

\begin{table*}[htb]
    \centering
    \begin{tabular}{|l|l|l|l|} \hline
    & \textbf{Feedforward}              & \textbf{Meta-NML}               & \textbf{Naive CNML} \\ \hline
    \textbf{Mazes (zigzag, spiral)} & 23.50s              & 39.05s         & 4hr 13min 34s              \\ \hline
    \textbf{Sawyer 2D Pusher} & 24.91s & 43.81 & 5hr 44min 25s \\ \hline
    \textbf{Sawyer Door} & 19.77s & 38.52s & 5hr 45min 00s \\ \hline
    \textbf{Sawyer 3D Pick} & 20.24s & 40.73s & 5hr 45min 00s \\ \hline
    \textbf{Ant Locomotion} & 37.15s & 73.72s & 4hr 47min 40s \\ \hline
    \textbf{Dexterous Hand Manipulation} & 48.37s & 69.97s & 4hr 53min 00s \\ \hline
    \end{tabular}
    \caption{Runtimes for completing a single epoch of RL according to Algorithm~\ref{alg:baycrl-alg}. We collect 1000 samples in the environment with the current policy for each epoch of training. The naive CNML runtimes are extrapolated based on the per-input runtime in the previous table, while the feedforward and meta-NML runtimes are averaged over 10 actual epochs of RL. These times indicate that naive CNML would be computationally infeasible to run in an RL algorithm, whereas meta-NML is able to achieve performance much closer to that of an ordinary feedforward classifier and make learning possible.}
    \label{tab:meta-nml-runtimes-per-epoch}
\end{table*}

\subsection{Experimental Details}

\begin{figure}[!h]
    \centering
    \includegraphics[width=0.3\textwidth]{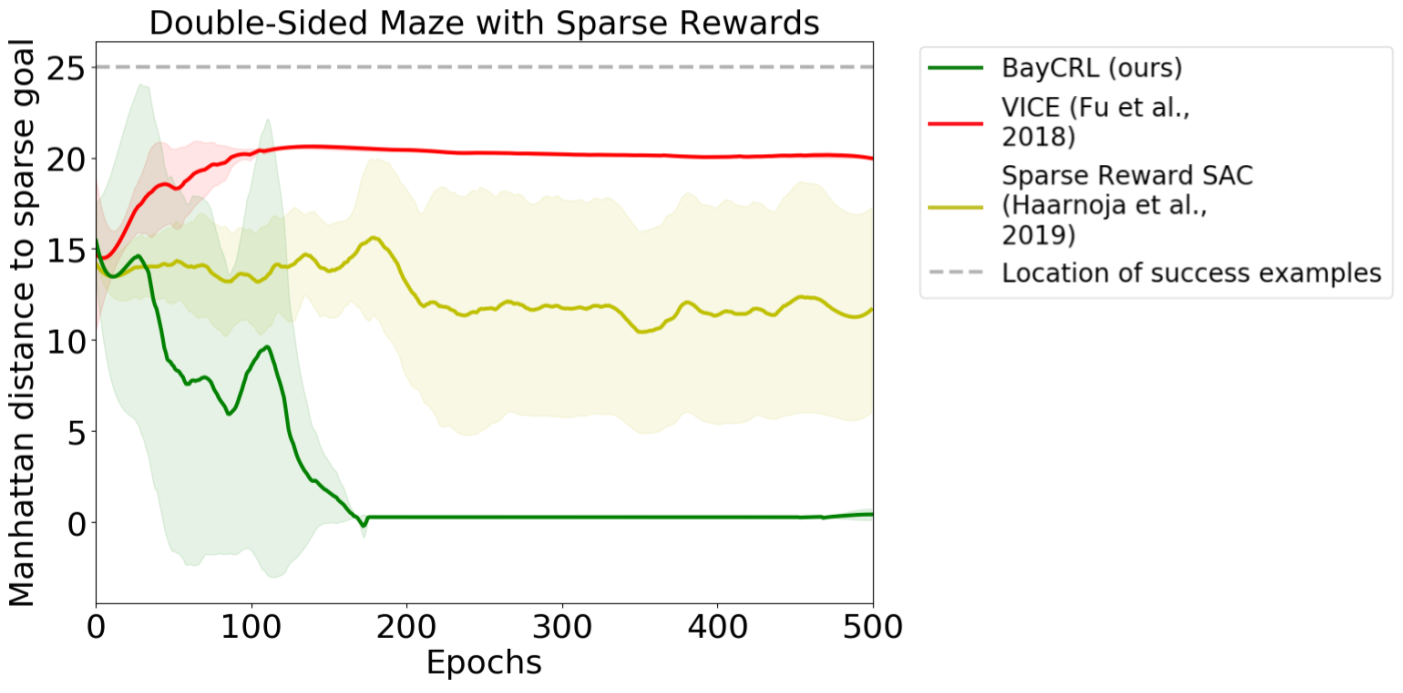}
    \caption{\footnotesize{Performance of \METHOD, VICE, and SAC with sparse rewards on a double-sided maze where some sparse reward states are not provided as goal examples. \METHOD is still able to find the sparse rewards, thus receiving higher overall reward, whereas ordinary classifier methods (i.e. VICE) move only towards the provided examples and thus are never able to find the additional rewards. Standard SAC with sparse rewards, also included for comparison, is generally unable to find the goals. The dashed gray line represents the location of the goal examples initially provided to both \METHOD and VICE.}}
    \label{fig:double-maze-plot}
\end{figure}

\subsubsection{Environments}
\label{sec:environment-details}
\textbf{Zigzag Maze and Spiral Maze:} These two navigation tasks require moving through long corridors and avoiding several local optima in order to reach the goal. For example, on Spiral Maze, the agent must not get stuck on the other side of the inner wall, even though that position would be close in L2 distance to the desired goal. On these tasks, a sparse reward is not informative enough for learning, while ordinary classifier methods get stuck in local optima due to poor shaping near the goal.

Both of these environments have a continuous state space consisting of the $(x, y)$ coordinates of the agent, ranging from $(-4, -4)$ to $(4, 4)$ inclusive. The action space is the desired velocity in the $x$ and $y$ directions, each ranging from $-1$ to $1$ inclusive.

\textbf{Sawyer 2D Pusher:} This task involves using a Sawyer arm, constrained to move only in the $xy$ plane, to push a randomly initialized puck to a fixed location on a table. The state space consists of the $(x, y, z)$ coordinates of the robot end effector and the $(x, y)$ coordinates of the puck. The action space is the desired $x$ and $y$ velocities of the arm.

\textbf{Sawyer Door Opening:} In this task, the Sawyer arm is attached to a hook, which it must use to open a door to a desired angle of 45 degrees. The door is randomly initialized each time to be at a starting angle of between 0 and 15 degrees. The state space consists of the $(x, y, z)$ coordinates of the end effector and the door angle (in radians); the action space consists of $(x, y, z)$ velocities.

\textbf{Sawyer 3D Pick and Place:} The Sawyer robot must pick up a ball, which is randomly placed somewhere on the table each time, and raise it to a fixed $(x, y, z)$ location high above the table. This represents the biggest exploration challenge out of all the manipulation tasks, as the state space is large and the agent would normally not receive any learning signal unless it happened to pick up the ball and raise it, which is unlikely without careful reward shaping.

The state space consists of the $(x, y, z)$ coordinates of the end effector, the $(x, y, z)$ coordinates of the ball, and the tightness of the gripper (a continuous value between 0 and 1). The robot can control its $(x, y, z)$ arm velocity as well as the gripper value. 

\textbf{Ant Locomotion:} In this task, the quadruped ant robot has to navigate from one end of a maze to the other. This represents a high dimensional action space of 8 dimensions, and a high dimensional state space of 15 dimensions as well. The state space consists of the center of mass of the object as well as the positions of the various joints of the ant, and the action space controls the torques on all the joints. 

\textbf{Hand Manipulation:} In this task, a 16 DoF robotic hand is mounted on a robot arm and has to reposition an object on a table. The task is challenging due to high dimensionality of the state and action spaces. The state space consists of the arm position, hand joint positions and object positions. In this task, we allow the classifier privileged access to the object position only, but provide the full state space as input to the policy. All the other baseline techniques are provided this same information as well (e.g. the classifier for VICE receives the object position as input).

\subsubsection{Ground Truth Distance Metrics}
\label{sec:distance-metrics}
In addition to the success rate plots in Figure~\ref{fig:comparisons}, we provide plots of each algorithm's distance to the goal over time according to environment-specific distance metrics. The distance metrics and success thresholds, which were used to compute the success rates in Figure~\ref{fig:comparisons}, are listed in the table on the next page.
\begin{table*}[htb]
\centering
    \begin{tabular}{|l|l|l|} \hline
    \textbf{Environment}              & \textbf{Distance Metric Used}               & \textbf{Success Threshold} \\ \hline
    Zigzag Maze              & Maze distance to goal         & 0.5               \\ \hline
    Spiral Maze              & Maze distance to goal         & 0.5               \\ \hline
    Sawyer 2D Pusher         & Puck L2 distance to goal           & 0.05              \\ \hline
    Sawyer Door Opening      & Angle difference to goal (radians) & 0.035             \\ \hline
    Sawyer 3D Pick-and-Place & Ball L2 distance to goal           & 0.06              \\ \hline
   Ant Locomotion & Maze distance to goal           & 5              \\ \hline
    Dexterous Manipulation & Object L2 distance to goal           & 0.06              \\ \hline
    \end{tabular}
\end{table*}

\begin{figure*}[!h]
    \centering
    \includegraphics[width=0.85\textwidth]{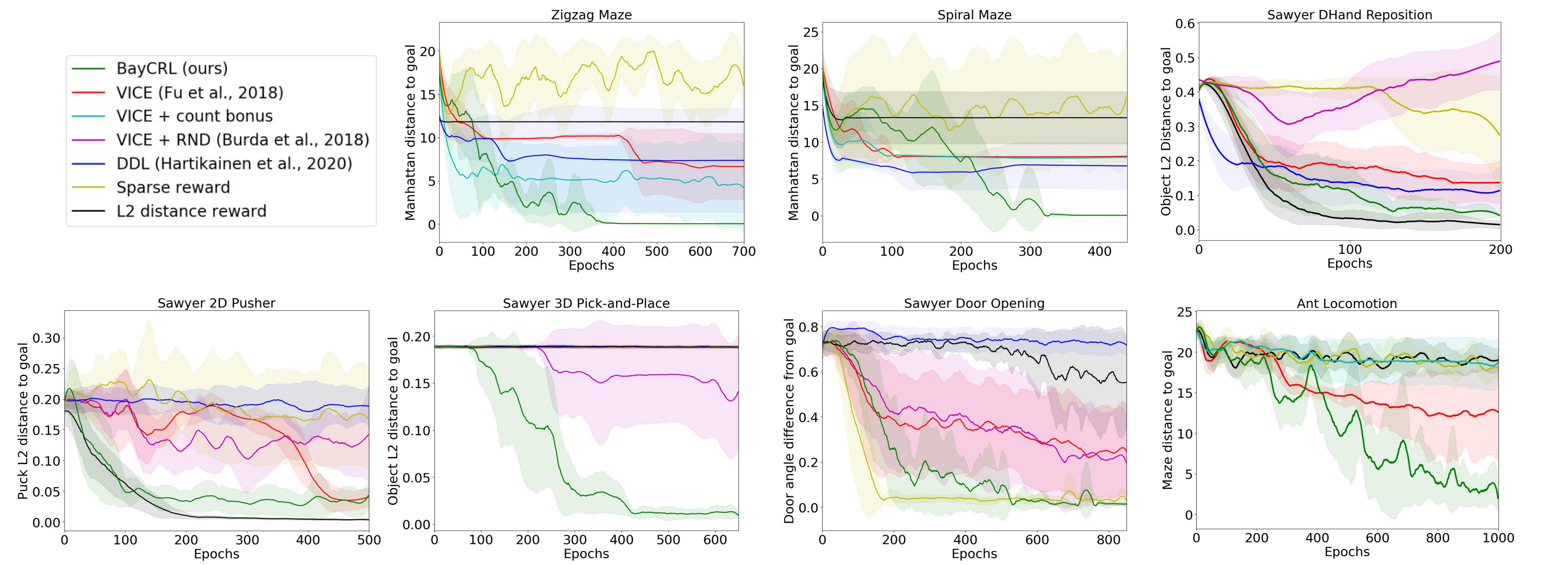}
    \caption{\footnotesize{Performance of \METHOD compared to other algorithms according to ground truth distance metrics. We note that while other algorithms seem to be making progress according to these distances, they are often actually getting stuck in local minima, as indicated by the success rates in Figure~\ref{fig:comparisons} and the visitation plots in Figure~\ref{fig:visitations}.}}
    \label{fig:distance-metrics}
\end{figure*}

\subsection{Additional Ablations}
\label{additional-ablations}
\subsubsection{Learning in a Discrete, Randomized Environment}
In practice, many continuous RL environments such as the ones we consider in \autoref{sec:experiments} have state spaces that are correlated at least roughly with the dynamics. For instance, states that are closer together dynamically are also typically closer in the metric space defined by the states. This correlation does not need to be perfect, but as long as it exists, \METHOD can in principle learn a smoothly shaped reward towards the goal.

However, even in the case where states are unstructured and completely lack identity, such as in a discrete gridworld environment, the CNML classifier would still reduce to providing an exploration-centric reward bonus, as indicated by \autoref{cnml-counts}, ensuring reasonable worst-case performance.

To demonstrate this, we evaluate \METHOD on a variant of the Zigzag Maze task where states are first discretized to a $16 \times 16$ grid, then "shuffled" so that the $xy$ representation of a state does not correspond to its true coordinates and the states are not correlated dynamically. \METHOD manages to solve the task, while a standard classifier method (VICE) does not. Still, \METHOD is more effective in the original state space where generalization is possible, suggesting that both the exploration and reward shaping abilities of the CNML classifier are crucial to its overall performance.

\begin{figure}[!h]
    \centering
    \includegraphics[width=0.4\textwidth]{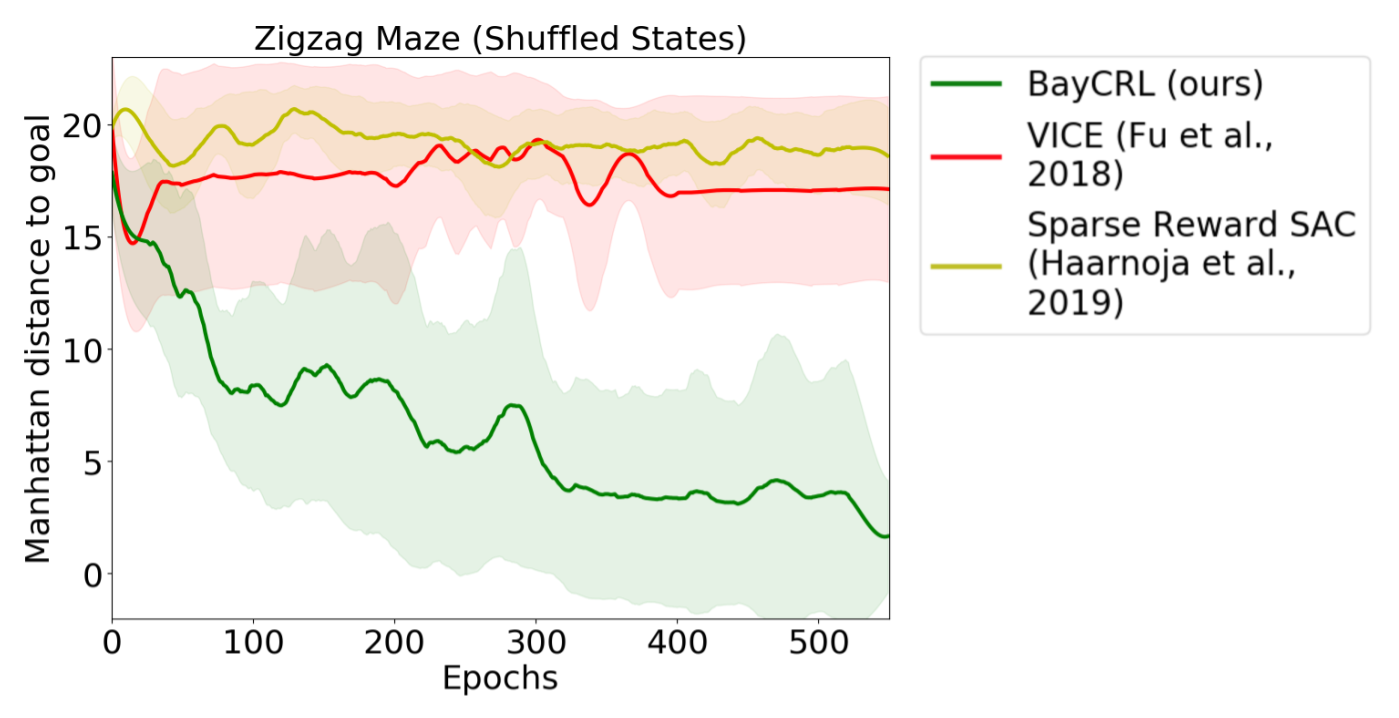}
    \caption{\footnotesize{Comparison of \METHOD, VICE, and SAC with sparse rewards on a discrete, randomized variant of the Zigzag Maze task. \METHOD is still able to solve the task on a majority of runs due to its connection to a count-based exploration bonus, whereas ordinary classifier methods (i.e. VICE) experience significantly degraded performance in the absence of any generalization across states.}}
    \label{fig:shuffled-maze-plots}
\end{figure}

\subsubsection{Finding "Hidden" Rewards Not Indicated by Success Examples}

\begin{figure}[!h]
    \centering
    \includegraphics[width=0.3\linewidth]{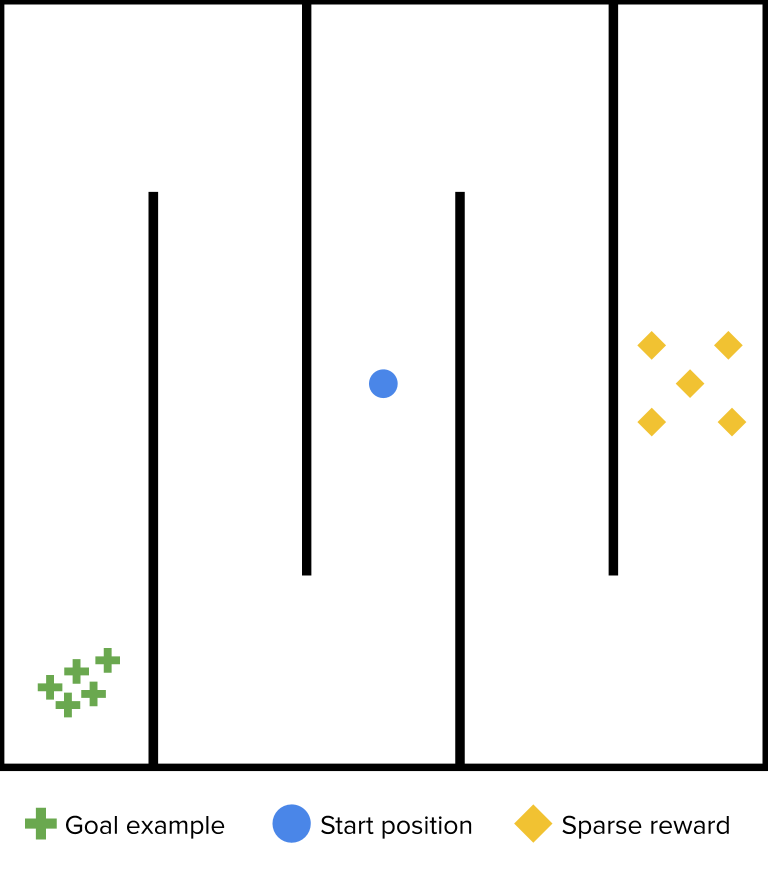}
    \caption{\footnotesize{Visualization of the Double-Sided Maze environment. Only the goal examples in the bottom left corner are provided to the algorithm.}}
    \vspace{-0.2cm}
    \label{fig:double-maze-layout}
\end{figure}

\begin{figure}[!h]
    \centering
    \vspace{-0.3cm}
    \includegraphics[width=0.8\linewidth]{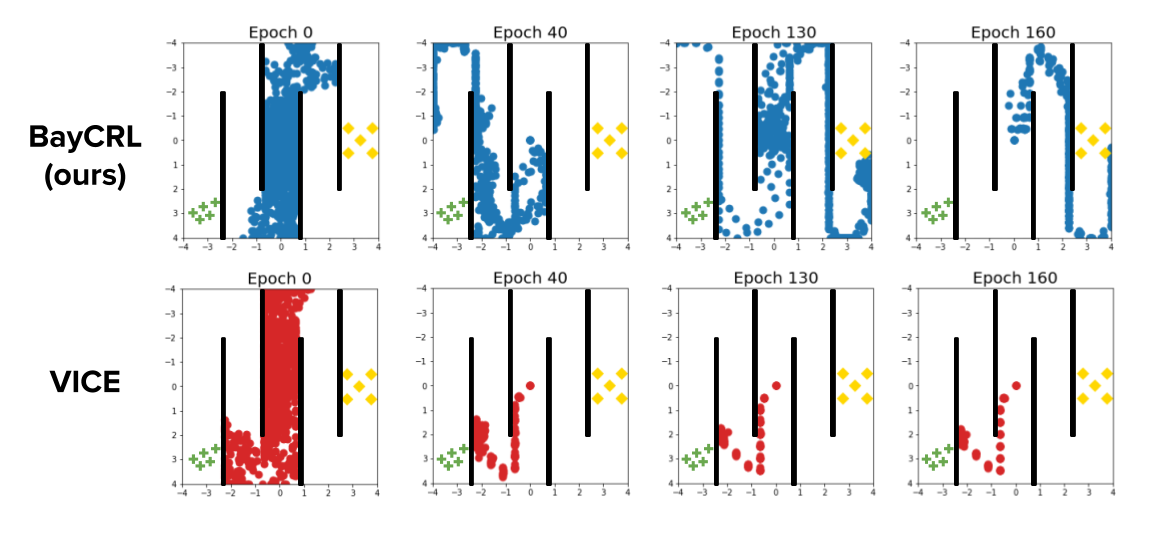}
    \caption{\footnotesize{Plot of visitations for \METHOD vs. VICE on the double-sided maze task. \METHOD is initially guided towards the provided goals in the bottom left corner as expected, but continues to explore in both directions, thus allowing it to find the hidden sparse rewards as well. Once this happens, it focuses on the right side of the maze instead because those rewards are easier to reach. In contrast, VICE moves only towards the (incomplete) set of provided goals on the left, ignoring the right half of the maze entirely and quickly getting stuck in a local optima.}}
    \label{fig:double-maze-visitations}
    \vspace{-0.4cm}
\end{figure}

The intended setup for \METHOD (and classifier-based RL algorithms in general) is to provide a  set of success examples to learn from, thus removing the need for a manually specified reward function. However, here we instead consider the case where a ground truth reward function exists which we do not fully know, and can only query through interaction with the environment. In this case, because the human expert has limited knowledge, the provided success examples may not cover all regions of the state space with high reward. 

An additional advantage of \METHOD is that it is still capable of finding these "unspecified" goals because of its built-in exploration behavior, whereas other classifier methods would operate solely based on the goal examples provided. To see this, we evaluate our algorithm on a two-sided variant of the Zigzag Maze with multiple goals, visualized in Figure~\ref{fig:double-maze-layout} to the right. The agent starts in the middle and is provided with 5 goal examples on the far left side of the maze; unknown to it, the right side contains 5 sparse reward regions which are actually closer from its initial position.

As shown in Figures~\ref{fig:double-maze-plot} and \ref{fig:double-maze-visitations}, \METHOD manages to find the sparse rewards while other methods do not. \METHOD, although initially guided towards the provided goal examples on the left, continues to explore in both directions and eventually finds the "hidden" rewards on the right. Meanwhile, VICE focuses solely on the provided goals, and gets stuck in a local optima near the bottom left corner.

\subsection{Hyperparameter and Implementation Details}
\label{appsec:hyperparams}
We describe the hyperparameter choices and implementation details for our experiments here. We first list the general hyperparameters that were shared across runs, then provide tables of additional hyperparameters we tuned over for each domain and algorithm. 

\textbf{Goal Examples:} For the classifier-based methods in our experiments (VICE and \METHOD), we provide 150 goal examples for each environment at the start of training. These are used as the pool of positive examples when training the success classifier.

\textbf{DDL Reward:} We use the version of DDL proposed in~\cite{ddl} where we provide the algorithm with the ground truth goal state $\mathbf{g}$, then run SAC with a reward function of $r(\mathbf{s}) = -d^\pi(\mathbf{s}, \mathbf{g})$, where $d^\pi$ is the learned dynamical distance function.

\subsubsection{General Hyperparameters}
\begin{table}[!h]
    \centering
    \begin{tabular}{| p{3cm}||p{4cm} |}
     \hline
     \textbf{SAC} & \\
     \hline
     Learning Rate & $3 \times 10^{-4}$\\
     Discount Factor $\gamma$ & $0.99$\\
     Policy Type & Gaussian\\
     Policy Hidden Sizes & $(512, 512)$\\
     Policy Hidden Activation & ReLU\\
     RL Batch Size & $1024$ \\
     Reward Scaling & $1$\\
     Replay Buffer Size & $500,000$\\
     Q Hidden Sizes & $(512, 512)$\\
     Q Hidden Activation & ReLU\\
     Q Weight Decay & $0$ \\
     Q Learning Rate & $3 \times 10^{-4}$\\
     Target Network $\tau$ & $5\times10^{-3}$ \\
     \hline
     \textbf{\METHOD} & \\
     \hline
     Adaptation batch size & 64\\
     Meta-training tasks per epoch & 128\\
     Meta-test set size & 2048\\
     \hline
     \textbf{VICE} & \\
     \hline
     Classifier Learning Rate & $1 \times 10^{-4}$\\
     Classifier Batch Size & $128$\\
     Classifier Optimizer & Adam\\
     RL Algorithm & SAC\\
     \hline
     \textbf{RND} & \\
     \hline
     Hidden Layer Sizes & $(256, 256)$\\
     Output Units & $512$\\
     \hline
    \end{tabular}
\caption{General hyperparameters used across all domains.}
\end{table}

\vfill\break

\subsubsection{Zigzag Maze Hyperparameters}
\begin{table}[!h]
    \centering
    \begin{tabular}{| p{3.5cm}||p{3.5cm} |}
     \hline
     \textbf{\METHOD} & \\
     \hline
     Classifier Hidden Layers & [(512, 512), \textbf{(2048, 2048)}] \\
     $\lambda_{dist}$ & $[\textbf{0.5}, 1]$\\
     $k_{query}$ & $\textbf{1}$\\
     \hline
     \textbf{VICE} & \\
     \hline
     $n_{\text{VICE}}$ & $[1, \textbf{2}, 10]$\\
     Mixup $\alpha$ & $[0, \textbf{1}]$\\
     Weight Decay $\lambda$ & $[0, \mathbf{5 \times 10^{-3}}]$ \\
     \hline
     \textbf{VICE+Count Bonus} & \\
     \hline
     $n_{\text{VICE}}$ & $[1, \textbf{2}, 10]$\\
     Mixup $\alpha$ & $[0, \textbf{1}]$\\
     Classifier reward scale & $[\textbf{0.25}, 0.5, 1]$\\
     Weight Decay $\lambda$ & $[\mathbf{0}, 5 \times 10^{-3}]$ \\
     \hline
     \textbf{DDL} & \\
     \hline
     $N_{d}$ & $[\mathbf{2}, 4]$\\
     Training frequency (every $n$ steps) & $[16, \mathbf{64}]$\\
     \hline
    \end{tabular}
    \caption{Hyperparameters we tuned for the Zigzag Maze task. Bolded values are what we use for the final runs in Section 6.}
\end{table}

\subsubsection{Spiral Maze Hyperparameters}
\begin{table}[!h]
    \centering
    \begin{tabular}{| p{3cm}||p{4cm} |}
     \hline
     \textbf{\METHOD} & \\
     \hline
     Classifier Hidden Layers & [(512, 512), \textbf{(2048, 2048)}] \\
     $\lambda_{dist}$ & $[\textbf{0.5}, 1]$\\
     $k_{query}$ & $\textbf{1}$\\
     \hline
     \textbf{VICE} & \\
     \hline
     $n_{\text{VICE}}$ & $[1, \textbf{2}, 10]$\\
     Mixup $\alpha$ & $[0, \textbf{1}]$\\
     Weight Decay $\lambda$ & $[0, \mathbf{5 \times 10^{-3}}]$ \\
     \hline
     \textbf{VICE+Count Bonus} & \\
     \hline
     $n_{\text{VICE}}$ & $[1, \textbf{2}, 10]$\\
     Mixup $\alpha$ & $[0, \textbf{1}]$\\
     Classifier reward scale & $[\textbf{0.25}, 0.5, 1]$\\
     Weight Decay $\lambda$ & $[\mathbf{0}, 5 \times 10^{-3}]$ \\
     \hline
     \textbf{DDL} & \\
     \hline
     $N_{d}$ & $[\mathbf{2}, 4]$\\
     Training frequency (every $n$ steps) & $[16, \mathbf{64}]$\\
     \hline
    \end{tabular}
    \caption{Hyperparameters we tuned for the Spiral Maze task. Bolded values are what we use for the final runs in Section 6.}
\end{table}

\newpage
\subsubsection{Ant Locomotion Hyperparameters}
\begin{table}[!h]
    \centering
    \begin{tabular}{| p{3cm}||p{4cm} |}
     \hline
     \textbf{\METHOD} & \\
     \hline
     Classifier Hidden Layers & [(512, 512), \textbf{(2048, 2048)}] \\
     $\lambda_{dist}$ & $[0.5, \textbf{1}, 1.5, 2]$\\
     $k_{query}$ & $\textbf{1}$\\
     \hline
     \textbf{VICE} & \\
     \hline
     $n_{\text{VICE}}$ & $[1, \textbf{2}, 10]$\\
     Mixup $\alpha$ & $[0, \textbf{1}]$\\
     Weight Decay $\lambda$ & $[0, \mathbf{5 \times 10^{-3}}]$ \\
     \hline
     \textbf{VICE+Count Bonus} & \\
     \hline
     $n_{\text{VICE}}$ & $[1, \textbf{2}, 10]$\\
     Mixup $\alpha$ & $[0, \textbf{1}]$\\
     Classifier reward scale & $[\textbf{0.25}, 0.5, 1]$\\
     Weight Decay $\lambda$ & $\mathbf{5 \times 10^{-3}}$ \\
     \hline
     \textbf{DDL} & \\
     \hline
     $N_{d}$ & $[2, \mathbf{4}]$\\
     Training frequency (every $n$ steps) & $[\mathbf{16}, 64]$\\
     \hline
    \end{tabular}
    \caption{Hyperparameters we tuned for the Ant Locomotion task. Bolded values are what we use for the final runs in Section 6.}
\end{table}

\subsubsection{Sawyer Push Hyperparameters}
\begin{table}[!h]
    \centering
    \begin{tabular}{| p{3.5cm}||p{3.5cm} |}
     \hline
     \textbf{\METHOD} & \\
     \hline
     Classifier Hidden Layers & [(512, 512), \textbf{(2048, 2048)}] \\
     $\lambda_{dist}$ & $[0.2, \textbf{0.6}, 1]$\\
     $k_{query}$ & $\textbf{1}$\\
     \hline
     \textbf{VICE} & \\
     \hline
     $n_{\text{VICE}}$ & $[1, 2, \textbf{10}]$\\
     Mixup $\alpha$ & $[0, \textbf{1}]$\\
     Weight Decay $\lambda$ & $[\textbf{0}, 5 \times 10^{-3}]$ \\
     \hline
     \textbf{VICE + RND} & \\
     \hline
     $n_{\text{VICE}}$ & $[1, 2, \textbf{10}]$\\
     Mixup $\alpha$ & $[0, \textbf{1}]$\\
     RND reward scale & $[\textbf{1}, 5, 10]$\\
     \hline
     \textbf{DDL} & \\
     \hline
     $N_{d}$ & $[\textbf{4}, 10]$\\
     Training frequency (every $n$ steps) & $[\textbf{16}, 64]$\\
     \hline
    \end{tabular}
    \caption{Hyperparameters we tuned for the Sawyer Push task. Bolded values are what we use for the final runs in Section 6.}
    \vspace{-0.5cm}
\end{table}

\vfill\break

\subsubsection{Sawyer Pick-and-Place Hyperparameters}
\begin{table}[!h]
    \centering
    \begin{tabular}{| p{3.5cm}||p{3.5cm} |}
     \hline
     \textbf{\METHOD} & \\
     \hline
     Classifier Hidden Layers & [\textbf{(512, 512)}, (2048, 2048)] \\
     $\lambda_{dist}$ & $[0.2, \textbf{0.6}, 1]$\\
     $k_{query}$ & $\textbf{1}$\\
     \hline
     \textbf{VICE} & \\
     \hline
     $n_{\text{VICE}}$ & $[1, \textbf{2}, 10]$\\
     Mixup $\alpha$ & $[0, \textbf{1}]$\\
     Weight Decay $\lambda$ & $[\textbf{0}, 5 \times 10^{-3}]$ \\
     \hline
     \textbf{VICE + RND} & \\
     \hline
     $n_{\text{VICE}}$ & $[1, \textbf{2}, 10]$\\
     Mixup $\alpha$ & $[0, \textbf{1}]$\\
     RND reward scale & $[\textbf{1}, 5, 10]$\\
     \hline
     \textbf{DDL} & \\
     \hline
     $N_{d}$ & $[4, \textbf{10}]$\\
     Training frequency (every $n$ steps) & $[\textbf{16}, 64]$\\
     \hline
    \end{tabular}
    \caption{Hyperparameters we tuned for the Sawyer Pick-and-Place task. Bolded values are what we use for the final runs in Section 6.}
\end{table}

\subsubsection{Sawyer Door Opening Hyperparameters}
\begin{table}[!h]
    \centering
    \begin{tabular}{| p{3.5cm}||p{3.5cm} |}
     \hline
     \textbf{\METHOD} & \\
     \hline
     Classifier Hidden Layers & [\textbf{(512, 512)}, (2048, 2048)] \\
     $\lambda_{dist}$ & $[0.05, 0.1, \textbf{0.25}]$\\
     $k_{query}$ & [1, $\textbf{2}]$\\
     \hline
     \textbf{VICE} & \\
     \hline
     $n_{\text{VICE}}$ & $[1, \textbf{5}, 10]$\\
     Mixup $\alpha$ & $[\textbf{0}, 1]$\\
     Weight Decay $\lambda$ & $[\textbf{0}, 5 \times 10^{-3}]$ \\
     \hline
     \textbf{VICE + RND} & \\
     \hline
     $n_{\text{VICE}}$ & $[1, \textbf{5}, 10]$\\
     Mixup $\alpha$ & $[\textbf{0}, 1]$\\
     RND reward scale & $[1, \textbf{5}, 10]$\\
     \hline
     \textbf{DDL} & \\
     \hline
     $N_{d}$ & $[4, \textbf{10}]$\\
     Training frequency (every $n$ steps) & $[\textbf{16}, 64]$\\
     \hline
    \end{tabular}
    \caption{Hyperparameters we tuned for the Sawyer Door Opening task. Bolded values are what we use for the final runs in Section 6.}
\end{table}

\newpage

\subsubsection{Dexterous Hand Repositioning Hyperparameters}
\begin{table}[!h]
    \centering
    \begin{tabular}{| p{3.5cm}||p{3.5cm} |}
     \hline
     \textbf{\METHOD} & \\
     \hline
     Classifier Hidden Layers & [(512, 512), \textbf{(2048, 2048)}] \\
     $\lambda_{dist}$ & $[0.2, \textbf{0.5}, 1]$\\
     $k_{query}$ & $\textbf{1}$\\
     \hline
     \textbf{VICE} & \\
     \hline
     $n_{\text{VICE}}$ & $[1, 2, \textbf{10}]$\\
     Mixup $\alpha$ & $[\textbf{0}, 1]$\\

     Weight Decay $\lambda$ & $[0, \mathbf{5 \times 10^{-3}}]$ \\
     \hline
     \textbf{VICE + RND} & \\
     \hline
     $n_{\text{VICE}}$ & $[1, \textbf{2}, 10]$\\
     Mixup $\alpha$ & $[\textbf{0}, 1]$\\
     RND reward scale & $[\textbf{1}, 5, 10]$\\
     \hline
     \textbf{DDL} & \\
     \hline
     $N_{d}$ & $[\textbf{4}, 10]$\\
     Training frequency (every $n$ steps) & $[\textbf{16}, 64]$\\
     \hline
    \end{tabular}
    \caption{Hyperparameters we tuned for the Dexterous Hand Repositioning task. Bolded values are what we use for the final runs in Section 6.}
\end{table}

\end{document}